%% file: main.tex
\documentclass{article}

\usepackage{amsthm}
\newtheorem{theorem}{Theorem}

\newtheorem{corollary}{Corollary}[theorem]
\newtheorem{lemma}{Lemma}[theorem]
\theoremstyle{definition}
\newtheorem{definition}{Definition}
\theoremstyle{definition}
\newtheorem{example}{Example}
\newtheorem{remark}{Remark}

\input{math_commands.tex}

\usepackage{braket}
\usepackage{wrapfig}
\usepackage{subcaption}

\newcommand{\EU}{\textsf{\upshape EU}}

\usepackage{tikzit}
\input{test.tikzstyles}

\usepackage{microtype}
\usepackage{graphicx}
\usepackage{booktabs} 
\usepackage{hyperref}
\usepackage{todonotes}

\usepackage[accepted]{icml2022}

\icmltitlerunning{Equivariant Quantum Graph Circuits}

\begin{document}

\twocolumn[
\icmltitle{Equivariant Quantum Graph Circuits}




\begin{icmlauthorlist}
\icmlauthor{Péter Mernyei}{ox,charm}
\icmlauthor{Konstantinos Meichanetzidis}{cqc}
\icmlauthor{{\.I}smail {\.I}lkan Ceylan}{ox}
\end{icmlauthorlist}

\icmlaffiliation{ox}{Department of Computer Science, University of Oxford, Oxford, UK.}
\icmlaffiliation{cqc}{Cambridge Quantum Computing and Quantinuum, Oxford, UK.}
\icmlaffiliation{charm}{Charm Therapeutics, London, UK.}

\icmlcorrespondingauthor{Péter Mernyei}{pmernyei@gmail.com}

\icmlkeywords{Graph representation learning, quantum machine learning}

\vskip 0.3in
]



\printAffiliationsAndNotice{} 

\begin{abstract}
We investigate quantum circuits for graph representation learning, and propose \emph{equivariant quantum graph circuits~(EQGCs)}, as a class of parameterized quantum circuits with strong relational inductive bias for learning over graph-structured data. Conceptually, EQGCs serve as a unifying framework for quantum graph representation learning, allowing us to define several interesting subclasses which subsume existing proposals. In terms of the representation power, we prove that the studied subclasses of EQGCs are universal approximators for functions over the bounded graph domain. This theoretical perspective on quantum graph machine learning methods opens many directions for further work, and could lead to models with capabilities beyond those of classical approaches. We empirically verify the expressive power of EQGCs through a dedicated experiment on synthetic data, and additionally observe that the performance of EQGCs scales well with the depth of the model and does not suffer from barren plateu issues.
\end{abstract}

\section{Introduction}
In recent years, the field of quantum computing has made significant steps towards practical usefulness, which has sparked increasing interest in many areas, including machine learning~\cite{perdomo2018qml, benedetti2019qml}. The growing field of quantum machine learning has since led to the development  of quantum analogues (or, generalizations) of many types of classical machine learning models, such as feedforward neural networks~\cite{Schuld2014TheQF}, convolutional neural networks~\cite{cong2019quantum} and graph neural networks~\cite{verdon2019quantum}.

Many existing quantum machine learning approaches seek advantages over classical methods by
exploiting the facts that quantum states and processes take place in the exponentially large Hilbert space,
that states can be superposed,
and that quantum amplitudes can interfere.
This, however, is still being explored: encoding useful quantum states efficiently and measuring them accurately are challenges that make straightforward speed-ups difficult~\cite{aaronson2015read}. Furthermore, since existing quantum devices are very limited, empirical benchmarks are often impossible at the scales where quantum methods might lead to a real advantage. Due to these difficulties, theoretical analysis plays a fundamental role, and recent works focusing on characterizing the capabilities and limitations of potential quantum models have shown significant results~\cite{schuld2021effect, liu2021rigorous, kubler2021inductive, goto2021universal}.

The goal of this paper is to establish a unifying framework for learning functions over graphs using parameterized quantum circuits and characterize the capabilities and limitations of various circuit classes within this framework. Classical graph machine learning techniques have been the backbone of many key developments in machine learning, since graphs are used to encode various forms of relational data, including knowledge graphs~\cite{bordes2011kgembedding}, social networks~\cite{zhang2018socialgraphs}, and importantly also molecules~\cite{moleculenet}, which are a particularly promising application domain of quantum computing due to their inherent quantum properties. 

Graph neural networks (GNNs)~\cite{gcn,gat} are prominent models for graph representation learning, as they encode desirable properties such as permutation invariance (resp., equivariance) relative to graph nodes, enabling a strong inductive bias~\cite{deepmind-survey}. While broadly applied, the expressive power of prominent GNN architectures, such as \emph{message-passing neural networks (MPNNs)}~\cite{mpnn}, is shown to be upper bounded by the 1-dimensional Weisfeiler-Lehman graph isomorphism test~\cite{gin,Morris2019WeisfeilerAL}. This motivated a large body of work aiming at more expressive models, including higher-order models~\cite{Morris2019WeisfeilerAL, Maron2019ProvablyPG}, as well as extensions of MPNNs with unique node identifiers~\cite{Loukas2020}, or with random node features~\cite{Sato2020RandomFS, Abboud2021TheSP}.

In this paper, we investigate quantum analogues of GNNs and make the following contributions:
\begin{itemize}
    \item We define criteria for quantum circuits to respect the invariances of the graph domain, leading to \emph{equivariant quantum graph circuits}~(EQGCs) (Section~\ref{sec-eqgc}).
    \item We define \emph{equivariant hamiltonian quantum graph circuits~(EH-QGCs)} and \emph{equivariantly diagonalizable unitary quantum graph circuits~(EDU-QGCs)} as special subclasses of EQGCs, and relate these classes to existing proposals from the literature, providing a unifying perspective for quantum graph representation learning (Section~\ref{sec-subclasses}).
    \item We characterize the expressive power of EH-QGCs and EDU-QGCs, proving that they are universal approximators for functions defined over the bounded graph domain. This result is achieved by showing a correspondence between EDU-QGCs and \emph{MPNNs enhanced with random node initialization} which are universal approximators over bounded graphs~\cite{Abboud2021TheSP}. Differently, our model does not require any extraneous randomization, and the result follows from the model properties (Section~\ref{sec-expressive}).
    \item We experimentally show that even simple EDU-QGCs go beyond the capabilities of popular GNNs, by empirically verifying that they can discern graph pairs, which are indiscernible by standard MPNNs~(Section~\ref{sec-experiment}).
\end{itemize}
The rest of this paper is organized as follows. We first discuss related work in the field of quantum machine learning in Section \ref{sec-relatedwork}. Following this, we give an overview of important methods and results in graph representation learning in Section \ref{sec-gnns}. After these preliminaries, we present our proposed framework and discuss important subclasses in Section \ref{sec-eqgc}, show our theoretical results on model expressivity in Section \ref{sec-expressive} and provide empirical evaluation in Section \ref{sec-experiment}. We finish with a discussion of our results and possible further directions in Section \ref{sec-discussion}.

All proof details and technical constructions can be found in the appendix of this paper.

\section{Related Work}
\label{sec-relatedwork}

The field of quantum machine learning includes a wide range of approaches. Early work had partial successes in speeding up important linear algebra subroutines~\cite{hhl}, but these methods usually came with caveats (e.g., requirements of the input being easy to prepare or being sparse, or approximate knowledge of the final state being sufficient) that made them hard to apply to large problem classes in practice~\cite{aaronson2015read}. Recent approaches tend to use quantum circuits to mimic or replace larger parts of classical techniques: \emph{quantum kernels} use a quantum computer to implement a fixed kernel function in a classical learning algorithm~\cite{schuld2019quantumkernels, liu2021rigorous}, while \emph{parameterized quantum circuits} (PQCs) use tunable quantum circuits as machine learning models in a manner similar to neural networks~\cite{perdomo2018qml,benedetti2019qml}. Lacking the possibility of standard backpropagation, there are alternative ways of calculating gradients~\cite{schuld2019quantumgradient}, and gradient-free optimization methods are also used~\cite{ostaszewski2021gradientfree}. In this paper, we focus on PQCs.

There is also a growing body of work on the capabilities and limitations of such models. \citeauthor{ciliberto2018quantum}~(\citeyear{ciliberto2018quantum}) and \citeauthor{kubler2021inductive}~(\citeyear{kubler2021inductive}) give rigorous results about when we can and cannot expect the inductive bias quantum of kernels to give them an advantage over classical methods; \citeauthor{servedio2004equivalences}~(\citeyear{servedio2004equivalences}) and \citeauthor{liu2021rigorous}~(\citeyear{liu2021rigorous}) demonstrate carefully chosen function classes that quantum kernels can provably learn more efficiently than any classical learner. PQCs have been harder to reason about due to their non-convex nature, but there have been important steps in showing conditions under which certain PQCs are universal function approximators over vector spaces~\cite{schuld2021effect, goto2021universal}, similarly to multi-layer perceptrons in the classical world~\cite{HORNIK1989359}. There has been also rigorous work on the PAC-learnability of the output distributions of local quantum circuits~\cite{hinsche2021learnability}.

For learning functions over graphs, the literature is rather sparse: there are some proposals based on PQCs~\cite{verdon2019quantum, zheng2021quantum} or based on kernels~\cite{henry2021quantum} supported by small-scale experiments, but there is generally a lack of formal justification for the particular model choices. In particular, we are not aware of any theoretical work on the capabilities of these models. We propose a framework unifying PQC models that build a circuit for each example graph in a structurally appropriate way when running inference, such as \citeauthor{verdon2019quantum}~(\citeyear{verdon2019quantum}), \citeauthor{zheng2021quantum}~(\citeyear{zheng2021quantum}), and \citeauthor{henry2021quantum}~(\citeyear{henry2021quantum}). Importantly, such PQCs are recently also used as a building block by \citeauthor{ai2022decompositional}~(\citeyear{ai2022decompositional}), who apply them to subgraphs, thereby requiring fewer qubits and enabling scaling to larger graphs. We discuss possible choices for circuit classes, and investigate their expressive power. 

It is worth nothing that there are also other quantum approaches that we do not cover, such as using edges primarily in classical pre- or post-processing steps of a PQC~\cite{chen2021hybrid}, or running a PQC for each node independently and optimizing a connectivity-based error term to ensure similarity of related nodes~\cite{beer2021quantum}

\section{Graph Neural Networks}
\label{sec-gnns}

GNNs can be dated back to earlier works of \citeauthor{scarselli2009}~(\citeyear{scarselli2009}) and \citeauthor{Gori2005}~(\citeyear{Gori2005}), which is followed by a rich line of work, leading to very popular GNNs~\cite{gcn,gin, gat,ggnn}. GNNs are designed to have a graph-based \emph{inductive bias}, i.e., the functions they learn should be invariant to the ordering of the nodes or edges of the graph, since the ordering is just a matter of representation and not a property of the graph. 
This includes \emph{invariant functions} that output a single value that should be unchanged on permuting nodes, and \emph{equivariant functions} that output a representation for each node, and this output is reordered consistently as the input is shuffled~\cite{gnnbook}. 
Formally, a function $f$ is \emph{invariant} over graphs if, for isomorphic graphs $\mathcal{G},\mathcal{H}$ it holds that $f(\mathcal{G})\,{=}\,f(\mathcal{H})$; a function $f$ mapping a graph $\mathcal{G}$ with vertices $V(\mathcal{G})$ to vectors ${\boldsymbol x\in\mathbb{R}^{|V(\mathcal{G})|}}$ is \emph{equivariant} if,  for every permutation $\pi$ of $V(\mathcal{G})$, it holds that ${f(\mathcal{G}^\pi)=f(\mathcal{G})^\pi}$. 

\emph{Message-passing neural networks} (MPNNs)~\cite{mpnn} are a popular and highly effective class of GNNs that iteratively update the representations of each node based on their local neighborhoods. In an MPNN, each node $v$ is assigned some initial state vector $\vh_v^{(0)}$ based on its features. This state vector is iteratively updated based on the current state of its neighbors $\mathcal N(v)$ and its own state, as follows:
\begin{equation*}
\label{eq-mpnn}
    \vh_v^{(k+1)} = \textsc{upd}^{(k)}\Big(\vh_v^{(k)}, \textsc{agg}^{(k)}\big(\{\!\!\{\vh_u^{(k)}~|~u \in \mathcal N(v) \}\!\!\}\big)\Big),
\end{equation*}

where $\{\!\!\{\cdot\}\!\!\}$ denotes a multiset, and $\textsc{agg}^{k}(\cdot)$ and $\textsc{upd}^{(k)}(\cdot)$ are differentiable functions. 

The specific choice for the aggregate and update functions varies across approaches, e.g., \emph{graph convolutional networks}~(GCNs)~\cite{gcn}, \emph{graph isomorphism networks}~(GINs) \cite{gin}, \emph{graph attention networks}~(GATs)~\cite{gat}, and \emph{gated graph neural networks}~ (GGNNs)~\cite{ggnn}. 

After several iteration (or, layers) have been applied, the final node embeddings are pooled to form a graph embedding vector to predict properties of entire graphs. The pooling often takes the form of simple averaging, summing or elementwise maximum.

\label{sec-background-expressivity}
\begin{figure}[t]
\centering
\includegraphics[width=\columnwidth]{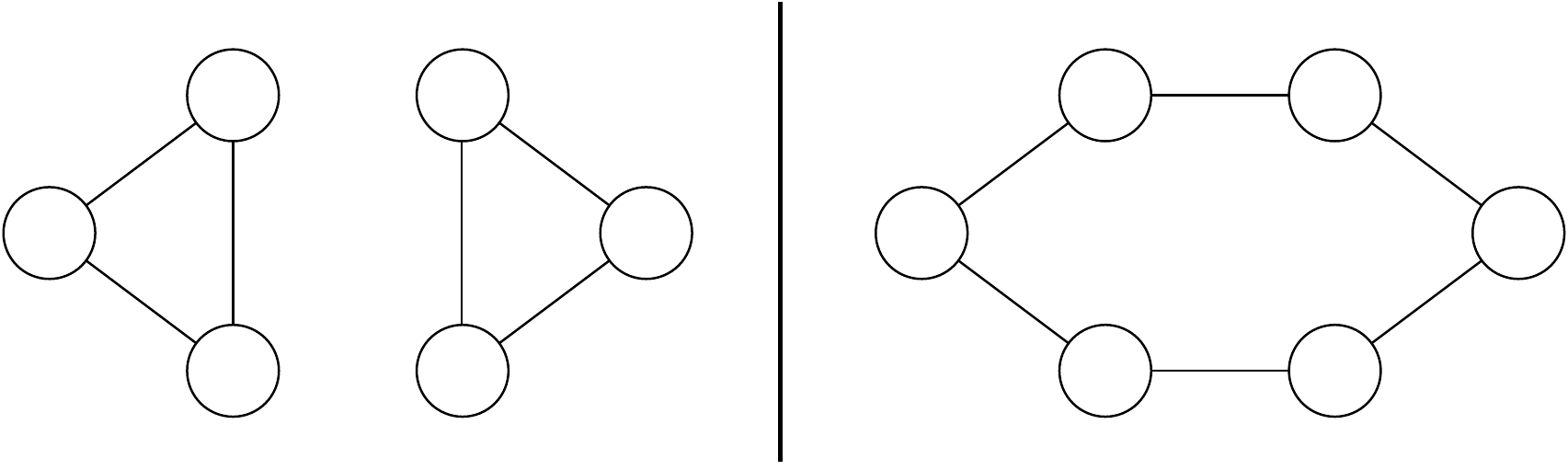}
\caption{Two graphs indistinguishable by 1-WL: $\gG_1$ consisting of two triangles~(left), and $\gG_2$ being a single 6-cycle~(right).}
\label{fig-twographs}
\end{figure}

One well-known limitation of MPNNs is their expressive power which is upper bounded by the 1-dimensional Weisfeiler-Lehman algorithm (1-WL) for graph isomorphism testing~\cite{gin,Morris2019WeisfeilerAL}. Considering a pair of 1-WL indistinguishable graphs, such as those shown in Figure \ref{fig-twographs}, any MPNN will learn the exact same representations for these graphs, yielding the same prediction for both, irrespectively of the target function to be learned. In particular, this means that MPNNs cannot learn functions such as counting cycles, or detecting triangles. 

The limitations in the expressive power of GNNs motivated a large body of work. \citeauthor{gin}~(\citeyear{gin}) proposed the graph isomorphism networks, as maximally expressive MPNNs, and showed this model to be as powerful as 1-WL, owing to its potential of learning injective aggregate-update functions. To break the expressiveness barrier, some approaches considered unique node identifiers~\cite{Loukas2020}, or random pre-set color features~\cite{DasoulasSSV20}, and alike, so as to make graphs discernible by construction (since 1-WL can distinguish graphs with unique node identifiers), but these approaches suffer in generalization. Other approaches are based on higher-order message passing~\cite{Morris2019WeisfeilerAL}, or higher-order tensors \cite{maron2018invariant, Maron2019ProvablyPG}, and typically have a prohibitive computational complexity, making them less viable in practice.

Rather recently, MPNNs enhanced with random node initialization~\cite{Sato2020RandomFS, Abboud2021TheSP} are shown to increase the expressivity without incurring a large computational overhead, and while preserving invariance properties in expectation. \citeauthor{Sato2020RandomFS}~(\citeyear{Sato2020RandomFS}) showed that such randomized MPNNs can detect any fixed substructure (e.g., a triangle) with high probability, and \citeauthor{Abboud2021TheSP}~(\citeyear{Abboud2021TheSP}) proved that randomized MPNNs are \emph{universal approximators} for functions over bounded graphs, building on an earlier logical characterization of MPNNs~\cite{Barcelo2020The}.
Intuitively, random node initialization assigns unique identifiers to different nodes with high probability and the model becomes robust via more sampling, leading to strong generalization. However, it is harder to train these models, since they need to see many different random labelings to eventually become robust to this variation. The extent of this effect can be mitigated by using fewer randomized dimensions~\cite{Abboud2021TheSP}.  

\section{Equivariant Quantum Graph Circuits}
\label{sec-eqgc}
In this section, we describe the class of models we are considering and formalize the requirement of respecting the graph structure in our definition of \emph{equivariant quantum graph circuits}. We then discuss two subclasses and their relation to each other as well as their relation to existing models.

\begin{figure}[ht]
\centering
\centerline{\includegraphics[width=0.95\columnwidth]{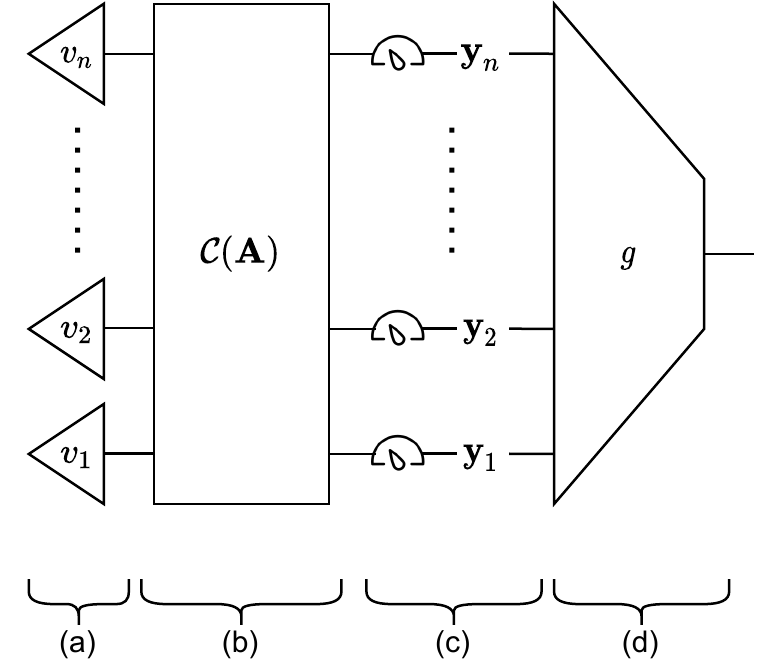}}
\caption{Overview of our model setup. \textbf{(a)} A product state is prepared based on individual nodes, \textbf{(b)} a parameterized circuit $\mC$ is applied based on the adjacency matrix $\mA$, \textbf{(c)} the nodes states are measured and \textbf{(d)} aggregated by some classical function $g$.}
\label{fig-modelpipeline}
\end{figure}

\subsection{Model Setup}
\label{sec-modelsetup}
Let $\sG^n$ be the set of graphs up to size $n$.
Consider a graph $\gG \in  \sG^n$, with adjacency matrix $\mA \in \mathbb B^{n \times n}$ and a node feature vector $\vx_i$ for each node $i \in \{1 \dots n\}$. We consider a broad class of models with the following simple structure, as shown in Figure \ref{fig-modelpipeline}:
\begin{enumerate}
    \item For each node with features $\vx_i$, a quantum state $\ket{v_i} = \ket{\rho(\vx_i)} \in \C^s$ is prepared via some fixed feature map $\rho(\cdot)$. The dimensionality of this state is $s = 2^q$ when using $q$ qubits per node.
    \item The node states are composed with the tensor product to form the product state $\ket{\vv} = \bigotimes_{i=1}^n \ket{v_i} \in \C^{s^n}$.
    \item We apply some circuit encoding a unitary matrix $\mC_{\vtheta}(\mA) \in \C^{s^n \times s^n}$, dependent on the adjacency matrix $\mA$ and tunable parameters $\vtheta$, to the initial state of the system.
    \item Each node state is measured in the computational basis, leading to a one-hot binary vector $\ket{y_i} \in \mathbb B^s$ for each node. Over the entire system, we measure any $\ket{\vy} = \bigotimes_{i=1}^n \ket{y_i} \in \mathbb B^{s^n}$ with probability ${P(\vy)=| \bra{\vy} \mC_{\vtheta}(\mA) \ket{\vv} |^2}$ as dictated by the Born rule. This means the probability of any specific measurement is given by the magnitude of a single element in the final state vector $\mC_{\vtheta}(\mA) \ket{\vv} \in \C^{n^s}$.
    \item These are aggregated by some permutation-invariant parameterized classical function $g_{\vtheta'}$ to provide a prediction $g_{\vtheta'}(\vy)$.
\end{enumerate}

While this setup rules out certain possibilities such as using mixed-state quantum computing with mid-circuit measurements, or somehow aggregating the node states inside the quantum circuit, it still leaves a broad and powerful framework that subsumes existing methods (as we will discuss in Section \ref{sec-subclasses}).

We do not consider details of how to design the classical aggregator $g_{\vtheta'}$ -- for questions of expressivity, we will simply assume that it is a universal approximator over multisets, which is known to be achievable by combining multi-layer perceptrons with sum aggregation~\cite{deepsets,gin}. The choice of the feature map $\rho$ does have to be made upfront, but our proofs all use simple constructions encoding the data in the computational basis.

Our focus is instead on the circuit $\mC_{\vtheta}(\mA)$, and how it should behave in order to interact well with the graph. As in the case of classical GNNs, we want to make sure the ordering of nodes and edges does not matter. In our case, this means that for any input, reordering the nodes and edges should reorder the probabilities of all measurements appropriately.

\begin{example}
With $n=3$ nodes represented by a single qubit each ($s=2$), the probability of observing some output $\bra{y_1y_2y_3}$ is $p = \bra{y_1y_2y_3}\mC_{\vtheta}(\mA)\ket{v_1v_2v_3}$. If we cycle the nodes around to form the input state $\ket{v_2v_3v_1}$, and also use an appropriately reordered adjacency matrix $\mA'$, we should find the probability of the reordered observation $\bra{y_2y_3y_1}\mC_{\vtheta}(\mA')\ket{v_2v_3v_1}$ to be $p$ as well.
\end{example}
This brings us to the definition of \emph{equivariant quantum graph circuits} (EQGCs): 
\begin{definition}
Let $\mA \in \sB^{n \times n}$ be an adjacency matrix, ${\mP \in \sB^{n \times n}}$ a permutation matrix representing a permutation $p$ over $n$ elements, and $\Tilde{\mP} \in \sB^{s^n \times s^n}$ a larger matrix that reorders the tensor product, mapping any $\ket{v_1}\ket{v_2}\dots\ket{v_n}$ with $\ket{v_i} \in \C^s$ to $\ket{v_{p(1)}}\ket{v_{p(2)}}\dots\ket{v_{p(n)}}$. 

An EQGC is an arbitrary parameterized function $\mC_{\vtheta}(\cdot)$ mapping an adjacency matrix $\mA \in \sB^{n \times n}$ to a unitary $\mC_{\vtheta}(\mA) \in \C^{s^n \times s^n}$ that behaves equivariantly for all $\vtheta$: 
\begin{equation}
\label{eq-eqgc-def}
    \mC_{\vtheta}(\mA) = \Tilde{\mP}^T \mC_{\vtheta}(\mP^T\mA\mP) \Tilde{\mP}
\end{equation}
\end{definition}
In the following sections, we will generally leave the parameter $\vtheta$, and sometimes also $\mA$, as implicit when they are clear from context. 

In accordance to our model setup, an EQGC $\mC_\vtheta(\cdot)$ represents a probabilistic model over graphs only when combined with a fixed feature map $\rho(\cdot)$ to prepare each node state, as well as measurement and classical aggregation $g_{\vtheta'}$ at the end of the circuit. Putting these together, we can formally speak of the capacity of EQGCs in representing functions. 

\begin{definition}
We say that a (Boolean or real) function $f$ defined on $\sG^n$ \emph{can be represented by an EQGC $\mC_\vtheta$ with error probability $\epsilon$} if there is some feature map $\rho$ and invariant classical aggregation function $g_\vtheta$, such that for any input graph $\gG \in \sG^n$ the model's output is $f(\gG)$ with probability $1-\epsilon$. In the special case, where $\eps=0$, we simply say that the function $f$ can be represented by an EQGC $\mC_\vtheta$.
\end{definition}

\begin{remark}[A note on directedness]
\label{sec-directedness}
Unlike many works on GNNs, our definition of EQGCs allows us to consider \emph{directed} graphs naturally, and this will also be true for the subclasses we consider later. Of course, we can still easily operate on undirected data by either adding edges in both directions, or placing extra restrictions on our models. For the purposes of expressivity, we will still focus on classifying graphs in the undirected case, as this is better explored in previous works on classical methods.
\end{remark}

\subsection{Subclasses of EQGCs}
\label{sec-subclasses}
Note that we cannot and should not aim to use all possible EQGCs as a model class. If we did, the prediction of our models on any graph would not restrict their behavior on other, non-isomorphic graphs in any way. This would not only make such a class impossible to characterize with a finite set of parameters $\vtheta$, but the models would also have no way to generalize to unseen inputs. Therefore EQGCs should be seen as a broad framework, and we investigate more restricted subclasses that do not have such problems.

We are particularly interested in subclasses that scale well with the number of nodes in a graph, so in the following sections we discuss approaches based on uniform single-node operations and two-node interactions at edges\footnote{We also considered the case, where $\mC_{\vtheta}(\cdot)$ depends only on the graph size rather than the adjacency matrix, and we report these findings in Appendix \ref{app-eqsc} as they are not central to our main results.}. All of the following models are parameterized by identical operations being applied for each node or for each edge, ensuring that a single model can efficiently learn about graphs of various sizes. It is also a useful starting point for ensuring equivariance, although as we will see, we also have to make sure that the ordering of these operations does not affect our results.

\subsubsection{Parameterization by Hamiltonians}
\label{sec-subclasses-qgcnns}
Operations on the quantum states of nodes or pairs of nodes can be easily represented as unitaries, but these are tricky to parameterize directly: e.g., a linear combination of unitaries is not unitary generally. One alternative is to use the fact that any unitary $\mU$ can be expressed using its Hamiltonian $\mH$, a Hermitian matrix of the same size such that $\mU = \exp(-i\mH)$. We can let the Hamiltonian depend linearly on the adjacency matrix, with Hermitian operators applied based on the structure of the graph:

\begin{definition}
An \emph{equivariant hamiltonian quantum graph circuit} (EH-QGC) is an EQGC given by a composition of finitely many layers $\mC_{\vtheta}(\mA) = \mL_{\vtheta_1}(\mA) \circ \dots \circ \mL_{\vtheta_k}(\mA)$, with each $\mL_{\vtheta_j}$ for $1 \leq j \leq k$ given as:
\begin{equation}
\label{eq-ehqgc}
    \mL_{\vtheta}(\mA) = \exp\Big(-i\big(\sum_{\mA_{jk}=1}\mH^{\text{(edge)}}_{j,k} + \sum_{i=1}^n\mH^{\text{(node)}}_i\big)\Big),
\end{equation}
where  $\mH^{\text{(edge)}}$ and $\mH^{\text{(node)}}$ are learnable Hermitian matrices over one and two-node state spaces comprising the parameter set $\vtheta$, and the indexing $\mH^{\text{(edge)}}_{j,k}, \mH^{\text{(node)}}_v$ refers to the same operators applied at the specified node(s) -- i.e., one EH-QGC layer is fully specified by a single one-node Hamiltonian and a single two-node Hamiltonian.
\end{definition}
This means that if the graph is permuted, the operators will be applied at changed positions appropriately. There is also no sequential ordering of operations in a summation, so the model is equivariant. For example, $\mH^{\text{(node)}}_3 = \mI \otimes \mI \otimes \hat \mH^{\text{(node)}} \otimes \mI$ in the case of $n=4$ nodes.

EH-QGCs is closely related to the approach taken by \citeauthor{verdon2019quantum}~(\citeyear{verdon2019quantum}) for their \emph{quantum graph convolutional neural network} (QGCNN) model as well as the parameterized \emph{quantum evolution kernel} of \citeauthor{henry2021quantum}~(\citeyear{henry2021quantum}). They both define operations in terms of Hamiltonians based on the graph structure, although their models are restricted to Hamiltonians of specific forms with fewer parameters. This helps for efficiently compiling small circuits (which is important on today's noisy machines), and allows better scaling to a larger number of qubits per node (which should be possible on future hardware). For our purposes, working with the broader class of arbitrary Hamiltonians lends itself better to theoretical analysis, and we leave it to future work to investigate circuit classes with better scaling in the number of qubits.

\subsubsection{Parameterization by Commuting Unitaries}
A similar, but more direct approach would be to consider two-node unitaries instead of Hamiltonians and apply a single learned unitary for each edge of the graph. As before, this ensures the number of operations scales linearly with the number of edges in a graph. This is also the approach taken by \citeauthor{zheng2021quantum}~(\citeyear{zheng2021quantum}), but we need to add extra conditions that they do not consider to ensure equivariance.

Specifically, we need to enforce that the order we apply these unitaries in does not matter. This gives us the following commutativity condition for a two-node unitary $\mU$:
\begin{equation}
\label{eq-commutativity}
  \input{commutativity.tikz}
\end{equation}

If the graphs are undirected, we should ensure the following to make sure the direction of the edge representation does not affect our predictions:
\begin{equation}
\label{eq-commutativity-undirected}
  \input{undirected-unitary.tikz}
\end{equation}

In the case of directed graphs, there are further conditions we need to satisfy instead of the above, which we detail in Appendix \ref{app-directed-unitaries}.

It is not clear whether we can parameterize the space of all such commuting unitaries, but we can focus on a subclass. 

\begin{definition}
\label{def-edu}
An \emph{equivariantly diagonalizable unitary} (EDU) is a unitary that can be expressed in the form $\mU = (\mV^\dagger \otimes \mV^\dagger)\mD(\mV \otimes \mV)$ for a unitary $\mV\in\C^{s \times s}$ and diagonal unitary $\mD \in \C^{s^2 \times s^2}$.
\end{definition}

Note that all unitaries can be diagonalized in the form $\mU = \mP^\dagger\mD\mP$ for some other unitary $\mP$ and diagonal unitary $\mD$. The above is simply the case when $\mP$ decomposes as $\mV \otimes \mV$ for one single-node unitary $\mV$. 
%

Using the facts that $\mI \otimes \mD$ is still a diagonal matrix and that diagonal matrices commute, we can see that equivariantly diagonalizable unitaries satisfy Equation \ref{eq-commutativity}:
\begin{equation}
\label{eq-edu-qgc}  
  \input{edu-qgc.tikz}
\end{equation}

Furthermore, a square matrix is unitary if and only if all of its eigenvalues (the diagonal elements of $\mD$) have absolute value $1$. We can therefore parameterize these unitaries by combining arbitrary single-node unitaries $\mV$ with diagonal matrices $\mD$ of unit modulus numbers\footnote{To add the inductive bias of undirected graphs, we can set $\mD\ket{e_1e_2} = \mD\ket{e_2e_1}$ for any computational basis vectors $\ket{e_1}, \ket{e_2}$, approximately halving the number of free parameters.}. 

This allows us to parameterize the following class of EQGCs:

\begin{definition}
\label{def-edu-qgc}
An \emph{equivariantly diagonalizable unitary quantum graph circuit} (EDU-QGC) is an EQGC expressed as a composition of \emph{node layers} $\mL_{\text{node}}$ and \emph{edge layers} $\mL_{\text{edge}}$ given as follows on a graph with node and edge sets $(\mathcal V, \mathcal E)$: 
\begin{align}
    \label{eq-eduqgc-node}
    \mL_{\text{node}} &= V^{\otimes |\mathcal V|}\\
    \label{eq-eduqgc-edge}
    \mL_{\text{edge}} &= \prod_{(j,k)\in\mathcal E}U_{jk}
\end{align}
\end{definition}
In short, we either apply the same single-node unitary to all nodes, or we apply the same EDU appropriately for each edge. Since both types of layers are equivariant by construction, so is their composition, hence EDU-QGCs are a valid EQGC class.

It can be shown that EDU-QGCs are a subclass of the Hamiltonian-based EH-QGCs discussed in Section \ref{sec-subclasses-qgcnns}. This is particularly useful for investigating questions of expressivity: we also get a result about the expressivity of EH-QGCs by showing the existence of EDU-QGC constructions representing some function.

\begin{theorem}
\label{thm-subclassing}
Any EDU-QGC can be expressed as an EH-QGC.
\end{theorem}

To show this result, we consider node layers and edge layers separately and show that both can be represented by one or more EH-QGC layers. We first prove the case for node layers, then diagonal edge layers; finally, we build on these two to prove the case for all edge layers, completing the proof. The details are provided in Appendix \ref{app-subclassing}.

\section{Expressivity Results}
\label{sec-expressive}
In this section, we analyse the expressivity of the EQGCs discussed in Section \ref{sec-subclasses}: Hamiltonian-based EH-QGCs and EDU-QGCs defined using commuting unitaries.

Quantum circuits operate differently from MPNNs and other popular GNN architectures, so one might hope that they are more expressive. Since current classical methods with high expressivity are either computationally expensive (like higher-order GNNs) or require a large number of training samples to converge (like GNNs with random node initialization), this could in principle lead to a form of quantum advantage with sufficiently large-scale quantum computers.

We first show that EDU-QGCs subsume MPNNs: a class of MPNNs, including maximally expressive architectures, can be `simulated' by a suitable EDU-QGC configuration. We then prove that they are in fact universal models for arbitrary functions on bounded-size graphs, building on prior results regarding randomized MPNNs. Since we have proven EDU-QGCs to be a subclass of EH-QGCs in Theorem \ref{thm-subclassing}, the results immediately follow for EH-QGCs as well.

\subsection{Simulating MPNNs}
Recall that MPNNs are defined via aggregate and update functions in Equation \ref{eq-mpnn}. In this section, we focus on MPNNs where the aggregation is of the form $\textsc{agg}^{(k)}(\{\!\!\{\vh_i\}\!\!\}) = \sum_i \vh_i$, which includes many common architectures.

\begin{remark}
We consider MPNNs node states with real numbers represented in fixed-point arithmetic. Although GNNs tend to be defined with uncountable real vector state spaces, these can be approximated with a finite set if the data is from a bounded set. 
\end{remark}

We show that EDU-QGCs can simulate MPNNs with sum aggregation in the following sense:

\begin{theorem}
\label{thm-simulate-mpnn} Any (Boolean or real) function over graphs that can be represented by an MPNN with sum aggregation, can also be represented by an EDU-QGC.
\end{theorem}
We prove this result, by giving an explicit construction to simulate an arbitrary MPNN with sum aggregation, detailed in Appendix \ref{app-simulate-mpnn}. In particular, our construction for Theorem~\ref{thm-simulate-mpnn} implies that for an MPNN with $k$ layers with an embedding dimensionality of $w$, with a fixed-point real representation of $b$ bits per real number, this EDU-QGC needs $(2k+1)wb$ qubits per node.

Since MPNNs with sum aggregation (e.g., GINs) can represent any function learnable by any MPNN~\cite{gin}, we obtain the following corollary to Theorem~\ref{thm-simulate-mpnn}:

\begin{corollary}
\label{corr-1wl-expressivity}
Any (Boolean or real) function that can be represented by any MPNN can also be represented by some EDU-QGC.
\end{corollary}

\subsection{Universal Approximation}
\label{sec-universality}
We build on results about randomization in classical MPNNs, discussed in Section \ref{sec-background-expressivity}~\cite{Sato2020RandomFS, Abboud2021TheSP}, to show that our quantum models are universal. 

We simulate classical models that randomize some part of the node state by putting some qubits into the uniform superposition over all bitstrings, then operating in the computational basis. Unlike in the classical case, where this randomization had to be explicitly added to extend model capacity, we can do this \emph{without modifying our model definition} -- our results apply to EDU-QGCs and their superclasses. Analogously to the universality of MPNNs with random features, this allows us to prove the following theorem:

\begin{theorem}
\label{thm-finite-real-functions}
For any real function $f$ defined over $\sG^n$, and any $\epsilon > 0$, an EDU-QGC can represent $f$ with an error probability $\epsilon$.
\end{theorem}

We cannot directly rely on the results of either \citeauthor{Abboud2021TheSP}~(\citeyear{Abboud2021TheSP}) or \citeauthor{Sato2020RandomFS}~(\citeyear{Sato2020RandomFS}): although our theorem is analogous to that of \citeauthor{Abboud2021TheSP}~(\citeyear{Abboud2021TheSP}), they used MPNNs extended with \emph{readouts at each layer}, which our quantum models cannot simulate. \citeauthor{Sato2020RandomFS}~(\citeyear{Sato2020RandomFS}) used MPNNs without readouts, but did not quite prove such a claim of universality. Therefore, we give a novel MPNN construction that is partially inspired by \citeauthor{Sato2020RandomFS}~(\citeyear{Sato2020RandomFS}), but relies solely on the results of \citeauthor{gin}~(\citeyear{gin}), and use it to show Theorem \ref{thm-finite-real-functions}.

Briefly, we use the fact that for bounded-size graphs individualized by random node features, a  GIN can in principle assign final node states that injectively depend on the isomorphism class of each node's connected component. These node embeddings can be pooled to give a unique graph embedding for each isomorphism classes of bounded graphs, which an MLP can map to any desired results. All of this can be simulated on an EDU-QGC, hence they are universal models. The details are given in Appendix \ref{app-universality}.

\section{Empirical Evaluation}
\label{sec-experiment}

\begin{figure*}[t!]
\begin{subfigure}{0.48\textwidth}
\centering
\scalebox{0.74}{ \input{circuit1.tikz} }
\caption{$\mC_\alpha$ applied to $\gG_1$.}
\end{subfigure}
\hfill
\begin{subfigure}{0.48\textwidth}
\centering
\scalebox{0.74}{ \input{circuit2.tikz} }
\caption{$\mC_\alpha$ applied to $\gG_2$.}
\end{subfigure}

\vspace{0.5cm}
\begin{subfigure}{0.48\textwidth}
\centering
\includegraphics[width=0.9\linewidth]{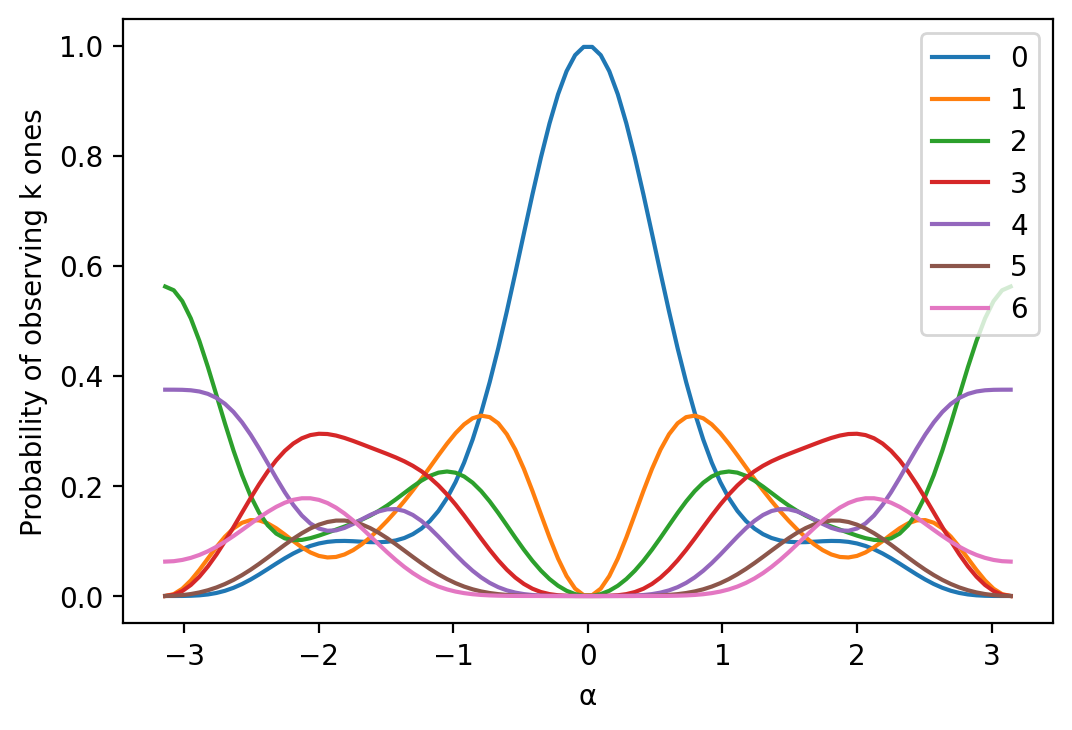} 
\caption{Measurement probabilities for $\gG_1$.}
\end{subfigure}
\hfill
\begin{subfigure}{0.48\textwidth}
\centering
\includegraphics[width=0.9\linewidth]{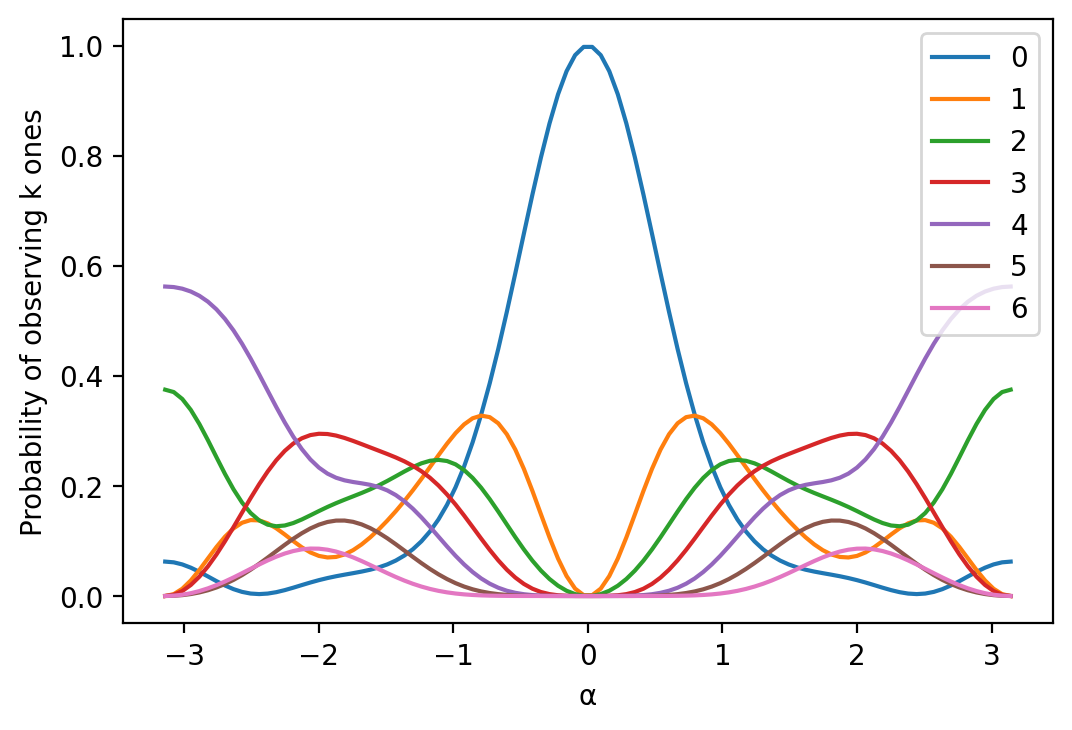} 
\caption{Measurement probabilities for $\gG_2$.}
\end{subfigure}

\caption{The two circuits in the experiment in top-to-bottom ZX-diagram notation, with the $\alpha$-box between white spiders representing a $CZ(\alpha)$ gate (a standard ZX-calculus shorthand~\cite{pqpbook}), followed by probabilities of observing given number of $\ket 1$s as a function of $\alpha \in [-\pi, \pi]$ for each circuit. The two distributions differ most visibly when $\alpha$ is near $\pm \pi$.}
\label{fig-experiment}
\end{figure*}

While our primary focus is theoretical, and it is challenging to execute experiments large enough to give interesting results, we performed two small experiments as well. We first look at a very restricted EDU-QGC model and observe that it can the graphs $\gG_1$ and $\gG_2$ with nontrivial probability (which is beyond the capabilities of MPNNs), and also reason about this simple case analytically. After this, we construct a small classification dataset of cycle graphs in a way that MPNNs could achieve no more than 50\% accuracy, and we successfully train deeper EDU-QGCs to high performance.

\subsection{Testing expressivity beyond 1-WL}

We performed a simple experiment to verify that EDU-QGC models can give different outputs for graphs that are indistinguishable by deterministic classical MPNNs.
As our inputs, we used the two graphs $\gG_1$ and $\gG_2$ shown in Figure \ref{fig-twographs} without node features (i.e., fixed initial node states in our quantum circuit), the simplest example where MPNNs fail. Our models should identify which graph is input. Using a \emph{single} qubit per node, we expect our accuracy to be better than $50\%$, but far from perfect.

\paragraph{Experimental setup.} To keep the experiment as simple as possible, we used a very simple subset of EDU-QGCs parameterized by a single variable $\alpha$, similar to \emph{instantaneous quantum polynomial} circuits~\cite{iqp}:
\begin{itemize}
    \item Each node state $\ket{v_i}$ is initialized as the $\ket{+}=H\ket{0}=\frac{1}{\sqrt 2}(\ket{0}+\ket{1})$ state on one node-qubit ($q=1$). By $H = \frac{1}{\sqrt 2}\big(\begin{smallmatrix}1 & 1 \\ 1 & -1\end{smallmatrix}\big)$ we denote the Hadamard gate.
    \item We apply an \emph{edge layer} as given by Equation \ref{eq-eduqgc-edge}, with a $CZ(\alpha) = \text{diag}(1,1,1,\exp(-i\alpha))$ gate as the applied unitary acting on two neighboring node-qubits.
    \item We apply a \emph{node layer} with an $H$ gate at each node.
    \item After a single measurement, we measure $k$ nodes as a $\ket 1$ state and $6-k$ as $\ket 0$. For each value of $k$, the aggregator $g_\alpha(\cdot)$ can map this to a different prediction.
\end{itemize}

Using ZX-diagram notation~\cite{pqpbook}, Figure \ref{fig-experiment}~(top) shows the circuits we get for our choice of $\mC$ in the case of $\gG_1$ and $\gG_2$. The probabilities of observing $k$ $\ket{1}$s for each graph and all possible values of $k$ as a function of our single parameter $\alpha$ is also shown in Figure \ref{fig-experiment}~(bottom). 
 
We find that as $\alpha$ gets near $\pm \pi$, the distributions of the number of $\ket 1$s measured do differ, and an accuracy of $0.625$ is achievable. This would naturally get better as we increase the number of qubits used, but this already shows an expressivity exceeding that of deterministic MPNNs. 

Further theoretical analysis of this setup is included in Appendix \ref{app-experiment-analysis}, using the ZX-calculus~\cite{pqpbook} to analytically derive the probabilities of all observations when applying this EDU-QGC with $\alpha=\pm \pi$ to arbitrary-size cycle graphs.

\subsection{Synthetic dataset of cycle graphs}

\begin{figure*}[t]

\begin{subfigure}{0.42\textwidth}
\centering
\includegraphics[width=\linewidth]{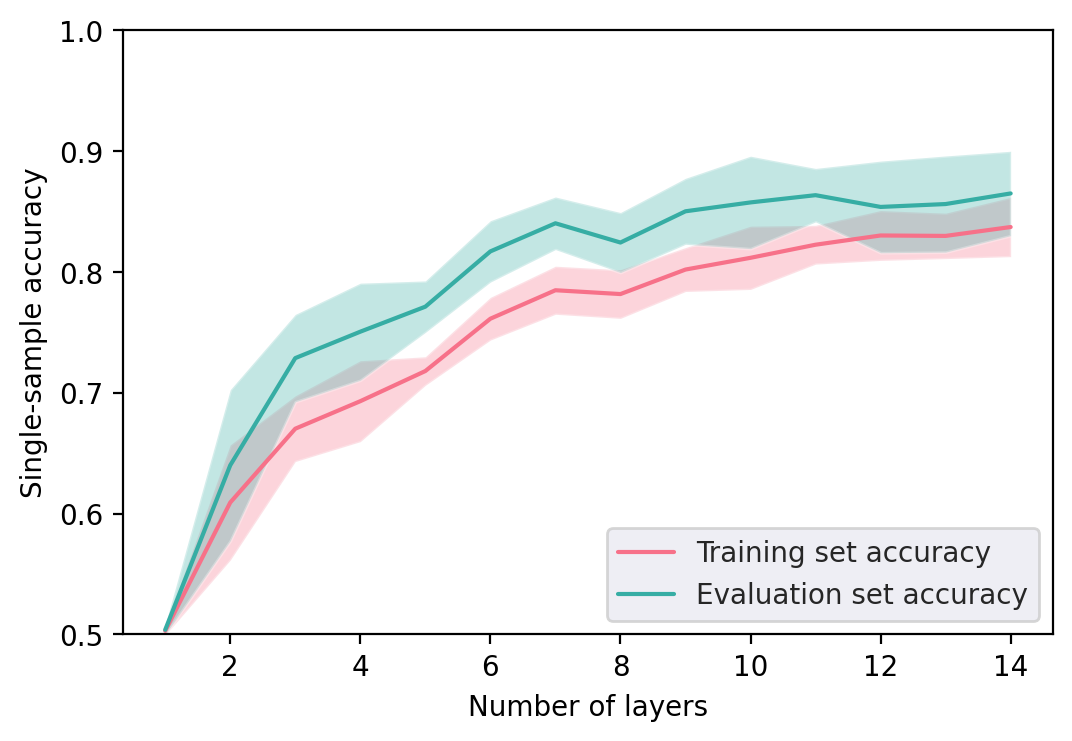} 
\caption{Single-sample accuracy by number of layers.}
\end{subfigure}
\hfill
\begin{subfigure}{0.42\textwidth}
\centering
\includegraphics[width=\linewidth]{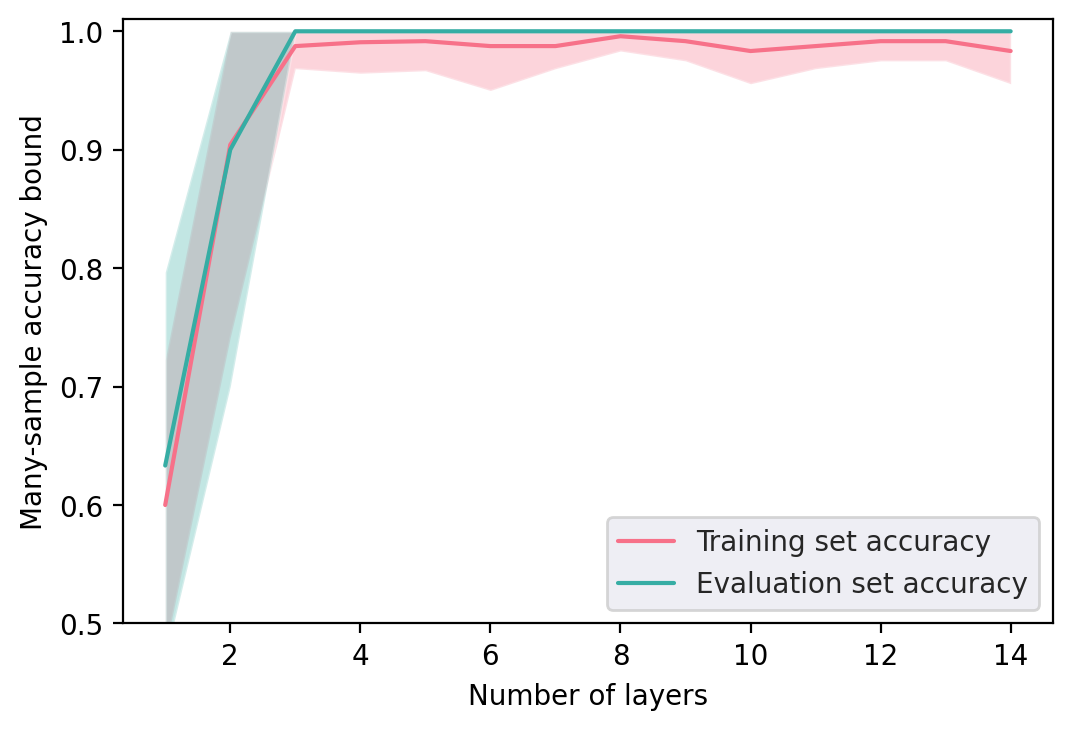} 
\caption{Highest many-sample accuracy by number of layers.}
\end{subfigure}

\caption{Accuracies of EDU-QGC models on the synthetic cycles dataset. The many-sample accuracy bound is calculated as the fraction of examples in the dataset where the model was correct with more than 50\% accuracy. Results are based on an average of 10 runs, with the shaded region representing standard deviation.\footnotemark[3]}
\label{fig-accuracies}
\end{figure*}

We created a synthetic dataset of 6 to 10-node graphs consisting of either a single cycle, or two cycles. The single-cycle graphs were oversampled to create two equally-sized classes for a binary classification task. 8-cycle graphs were reserved for evaluation, while all others were used for training.

We trained EDU-QGC models of various depths with a single qubit per node on this dataset. Each node state was initialized as $\ket{+}=\frac{1}{\sqrt 2}$, then an equal number $k\in \{1, \dots, 14\}$ general node and edge layers were applied alternatingly. After measurement, the fraction of observed $\ket 1$s was used to predict the input's class through a learnable nonlinearity. Exact probabilities of possible outcomes were calculated, and the Adam optimizer was used to minimize the expected binary cross-entropy loss for $100$ epochs, with an initial learning rate of $0.01$ and an exponential learning rate decay with coefficient $0.99$ applied at each epoch.

Results are shown in Figure \ref{fig-accuracies}. We report the one-sample accuracy (the average probability of a correct prediction across the dataset), and the highest achievable many-sample accuracy (the fraction of the dataset where a model was right with at least 50\% probability). 
Importantly, we observe a consistent benefit from increasing depth, in contrast with the oversmoothing problems of GNNs~\cite{Oversmoothing}. We also did not experience any issues with the near-zero gradients or `barren plateaus' that make it challenging to optimize many PQC models~\cite{mcclean2018barren}. The model was also able to fit this dataset with a very small number of parameters: after accounting for redundancy, the model contains only 6 real-valued degrees of freedom for each pair of node and edge layers (see Appendix~\ref{app-parameter-count}).

Interestingly, the model performs better on the evaluation set than the training set. This is due to the fact that it is hard for the model to reliably correctly classify 9 and 10-node graphs containing two cycles when these contain subgraphs that are in the one-cycle class. For example, the model associates a high number of measured $\ket 1$s with single-cycle graphs, then a 6-cycle will lead to many $\ket{1}$s. Since a disjoint union of a 6-cycle and a 3-cycle contains this subgraph, it will also have a relatively high fraction of $\ket 1$s, leading to an incorrect prediction. Clearly, this would not be an issue if more qubits per node could be used (which may be feasible in future): the size of a cycle could be encoded exactly in the larger set of possible observations, and this could be easily aggregated invariantly to count the number of cycles. \footnotetext[3]{One of 10 runs with was dropped as an outlier in the case of 4 layers. Through some unlucky initialization, the model failed to learn anything and stayed at 50\% accuracy in this single run.}

\section{Conclusions, Discussions, and Outlook}
\label{sec-discussion}

In this paper, we proposed equivariant quantum graph circuits, a general framework of quantum machine learning methods for graph-based machine learning, and explored possible architectures within that framework. Two subclasses, EH-QGCs and EDU-QGCs, were proven to have desirable theoretical properties: they are universal for functions defined up to a fixed graph size, just like randomized MPNNs. Our experiments were small-scale due to the computational difficulties of simulating quantum computers classically, but they did confirm that the distinguishing power of our quantum methods exceeds that of deterministic MPNNs.

By defining the framework of EQGCs and their subclasses, many questions can be raised that we did not explore in this paper. EDU-QGCs and EH-QGCs have important limitations: using arbitrary node-level Hamiltonians or unitaries allowed us to show expressivity results, but they are not feasible to scale to a large number of qubits per node, since the space of parameters grows exponentially. Perhaps a small number of qubits will already turn out to be useful, but EQGC classes with better scalability to large node states should also be investigated. 

There are also design choices beyond the EQGC framework that might be interesting. For example, rather than measuring only at the end of the circuit, mid-circuit measurements and quantum-classical computation might offer possibilities that we have not analyzed.

Ultimately, the biggest questions in the field of quantum computing are about quantum advantage: what useful tasks can we expect quantum computers to speed up, and what kind of hardware do these applications require? Recent work on the theoretical capabilities of quantum machine learning architectures is already contributing to this: it has been shown that we can carefully engineer artificial problems that provably favor quantum methods~\cite{kubler2021inductive,google-supremacy,liu2021rigorous}, but this is yet to be seen for practically significant problem classes. At the same time, there are convincing arguments that quantum computers will be useful for computational chemistry tasks such as simulating molecular dynamics, where EQGCs  could be useful, which is a direction worth exploring.

\bibliography{refs}
\bibliographystyle{icml2021}

\clearpage
\appendix

\section{Details on Commuting Unitaries}
\subsection{Commutativity Conditions in the Directed Case}
\label{app-directed-unitaries}
If our we are applying a two-node unitary for each edge of a directed graph, Equation \ref{eq-commutativity-undirected} need not apply, but \ref{eq-commutativity} is also not sufficient in itself, since we need to consider cases where the unitary might be applied in different directions. Specifically, we need to ensure the following extra conditions: 
\begin{equation}
\label{eq-commutativity-directed-1}
  \input{commutativity-directed-1.tikz}
\end{equation}
\begin{equation}
\label{eq-commutativity-directed-2}
  \input{commutativity-directed-2.tikz}
\end{equation}
\begin{equation}
\label{eq-commutativity-directed-3}  
  \input{commutativity-directed-3.tikz}
\end{equation}

Of course, such a directed unitary can also be used for directed graphs by applying it in both directions: in fact, if Equation \ref{eq-commutativity-directed-3} is satisfied, this composition itself satisfies the undirected Equation \ref{eq-commutativity-undirected}:
\begin{equation*}
  \input{undirected-from-directed-1.tikz}
\end{equation*}
\begin{equation}
  \input{undirected-from-directed-2.tikz}
\end{equation}
\begin{equation*}
  \input{undirected-from-directed-3.tikz}
\end{equation*}
In the directed case, the commutativity conditions are also satisfied, since $\mV \otimes \mV$ and $\mV^\dagger \otimes \mV^\dagger$ commute with the swap, and then analogous derivations to the undirected case apply.


%


\subsection{Proof of Theorem~\ref{thm-subclassing}: EDU-QGCs are a subclass of EH-QGCs}
\label{app-subclassing}
To prove this result, we initially consider EDU-QGC node layers and EDU-QGC edge layers separately and show that both can be represented by (one or more) EH-QGC layers, and afterwards combine these layers to show EH-QGCs can represent any EDU-QGC.

The proof is structured as follows: we first prove the case for node layers (Lemma \ref{lemma-eduqgc-nodes}), then diagonal edge layers (Lemma \ref{lemma-eduqgc-diagedges}); and, finally, we build on these two to prove the case for all edge layers (Lemma \ref{lemma-eduqgc-edges}), completing the proof of Theorem~\ref{thm-subclassing}.

\begin{lemma}
\label{lemma-eduqgc-nodes}
Any node layer $\mL_{\text{node}} = \mV^{\otimes |\mathcal V|}$ (as defined in Equation \ref{eq-eduqgc-node}) can be expressed as an EH-QGC layer.
\end{lemma}
\begin{proof}
Let $|\mathcal V|=n$ and let $\mR$ be the Hamiltonian for $\mV$. Then, $\mH = \mR^{\otimes n} = \sum_{v \in \mathcal V}\mR_v$ is an appropriate EH-QGC Hamiltonian (of the form defined in Equation \ref{eq-ehqgc}). 
We can easily show $\mH$ is then the Hamiltonian for the EDU-QGC layer $\mV^{\otimes n}$: 
\begin{align*}
    \exp(-i\mH) &= \sum_{k=0}^\infty \frac{(-i\mH)^k}{k!}\\
                &= \sum_{k=0}^\infty \frac{(-i)^k(\mR^{\otimes n})^k}{k!}\\
                &= \sum_{k=0}^\infty \frac{((-i\mR)^k)^{\otimes n}}{k!}\\
                &= \Big(\sum_{k=0}^\infty \frac{((-i\mR)^k)}{k!}\Big)^{\otimes n}\\
                &= \exp(-i\mR)^{\otimes n}\\
                &= \mV^{\otimes n}
\end{align*}
\end{proof}
\begin{lemma}
\label{lemma-eduqgc-diagedges}
Any diagonal edge layer $\mL_{\text{diag}} = \prod_{(j,k)\in\mathcal E}\mD_{jk}$, with a \emph{diagonal} unitary applied for each edge, can be expressed as an EH-QGC layer.
\end{lemma}
\begin{proof}
A diagonal unitary $\mD$ has a diagonal Hamiltonian $\mR$, where $\mD_{jj}=\exp(-i\mR_{jj})$. Using the fact that $\exp(\mA)\exp(\mB)=\exp(\mA+\mB)$ for commuting matrices $\mA$ and $\mB$, and that all diagonal matrices commute, we will derive that applying the Hamiltonian $\mR$ for each edge at the same time has the effect of applying the unitary $\mD$ for each. 

Consider two edges $\{(v_1,u_1), (v_2,u_2)\}$. The overall unitary we apply, with implicit identities on all other nodes, is
\begin{align*}
    \mD_{v_2u_2}\mD_{v_1u_1} &= \exp(-i\mR_{v_2u_2})\exp(-i\mR_{v_1u_1})\\
    &= \exp(-i(\mR_{v_2u_2} + \mR_{v_1u_1}))
\end{align*}
This generalizes easily to $n$ nodes:  the Hamiltonian of the overall unitary is $\sum_{(j,k)\in\mathcal E} \mR_{jk}$ as required. 
\end{proof}

\begin{lemma}
\label{lemma-eduqgc-edges}
Any edge layer $\mL_{\text{edge}} = \prod_{(j,k)\in\mathcal E}\mU_{jk}$ (as defined in Equation \ref{eq-eduqgc-edge}), with any equivariantly diagonizable unitary $\mU$, can be expressed as an EH-QGC layer.
\end{lemma}
\begin{proof}
This relies on Lemmas \ref{lemma-eduqgc-nodes} and \ref{lemma-eduqgc-diagedges}. We can show that a layer of equivariantly diagonalizable unitaries can expressed as a layer of diagonal unitaries sandwiched between two layers of single-node unitaries. Each of these can be represented as an EH-QGC layer by the previous lemmas, therefore giving us a 3-layer EH-QGC construction for this statement.

Consider an equivariantly diagonalizable unitary $\mU=(\mV^\dagger \otimes \mV^\dagger)\mD(\mV \otimes \mV)$ applied for each edge in a layer $\prod_{(j,k)\in\mathcal E}\mU_{jk}$. From the perspective of each node involved in edges, this decomposes as follows:
\begin{itemize}
    \item a single-node unitary $\mV$
    \item some number of two-node diagonal matrices separated by $\mV^\dagger\times\mV=\mI$, which can be ignored
    \item a single-node unitary $\mV^\dagger$
\end{itemize}
For nodes that are not part of any edge, we have the identity matrix that can be written as $\mV^\dagger \times \mV$. So we can rewrite the layer:
\begin{equation*}
    \prod_{(j,k)\in\mathcal E}\mU_{jk} = \Big((\mV^\dagger)^{\otimes n}\Big)\Big(\prod_{(j,k)\in\mathcal E}\mD_{jk}\Big)\Big(\mV^{\otimes n}\Big)
\end{equation*}
This is of the 3-layer form we discussed, proving the lemma.
\end{proof}

Given these, we can prove the result:
\begin{proof}[Proof of Theorem \ref{thm-subclassing}]
Putting together Lemma \ref{lemma-eduqgc-nodes} and \ref{lemma-eduqgc-edges} completes the proof: both types of EDU-QGC layers given by Equations \ref{eq-eduqgc-node} and \ref{eq-eduqgc-edge} can be represented by one or more EH-QGC layers, so a sequence of EH-QGC layers can represent any EDU-QGC.
\end{proof}

\section{Proofs of Expressivity Results}
\label{app-proofs}

\subsection{Proof of Theorem~\ref{thm-simulate-mpnn}: Simulating MPNNs}
\label{app-simulate-mpnn}

We give an explicit construction to simulate an arbitrary MPNN with sum aggregation, i.e., an arbitrary MPNN where the aggregation is of the form: 
\[
\textsc{agg}^{(k)}(\{\!\!\{\vh_i\}\!\!\}) = \sum_i \vh_i.
\] 
The node states will be conceptually split into registers representing fixed-point real numbers in two's complement in the computational basis. We first need to establish that we can perform addition on these registers using unitary transformations.

\begin{lemma}
\label{lemma-add-integers}
Consider two node states with two registers each, storing unsigned integers: $\ket{a_1,a_2} \otimes \ket{b_1,b_2}$, with $a_i,b_i \in \{0, \dots, 2^b-1\}$ for some $b$. Let $\mU$ map $\ket{a_1,b_1} \otimes \ket{a_2,b_2}$ to $\ket{a_1,b_1+a_2} \otimes \ket{a_2,b_2+a_1}$, with standard overflowing addition. Then, $\mU$ is an equivariantly diagonalizable unitary and satisfies the undirected symmetry condition in Equation \ref{eq-commutativity-undirected}.
\end{lemma}

\begin{proof}
Let $\mS_a$ be a single-node unitary that increments integers encoded in the computational basis by $a$. Note that $\mS_a = \mS_1^a$. Diagonalize $\mS_1$ as $\mV^\dagger\mD\mV$, then $\mS_a = (\mV^\dagger\mD\mV)^a = \mV^\dagger\mD^a\mV$.

Now $\mU$ can be represented by applying $\mV$ to the second register of each node, conditionally applying $\mD$ to the the second register of each node some number of times depending on the value of the first register, and finally applying $\mV^\dagger$ to the second registers. The controlled application of a diagonal matrix is still diagonal, so this decomposition diagonalizes $\mU$ equivariantly with $(\mI \otimes \mV)^{\otimes 2}$.

The undirected symmetry Equation \ref{eq-commutativity-undirected} can be seen easily from the definition of $\mU$: swapping $a_1$ with $a_2$ and $b_1$ with $b_2$ results in swapping the values in the output.
\end{proof}

\begin{lemma}
\label{corr-add-reals}
Consider two node states with two registers each, storing fixed point unitaries in two's complement: $\ket{a_1,a_2} \otimes \ket{b_1,b_2}$, with $a_i,b_i \in \{(-2^{b-1}+1)\times 2^{-k}, \dots, 2^{b-1}\times 2^{-k}\}$ for some $b,k$. Let $\mU$ map $\ket{a_1,b_1} \otimes \ket{a_2,b_2}$ to $\ket{a_1,b_1+a_2} \otimes \ket{a_2,b_2+a_1}$, with standard overflowing addition. Then, $\mU$ is an equivariantly diagonalizable unitary.
\end{lemma}

\begin{proof}
As far as the bit-level operations are concerned, this is exactly the same as Lemma \ref{lemma-add-integers}: with two's complement, standard overflowing addition of unsigned integers can represent addition of signed integers, and fixed point reals are essentially integers interpreted with a multiplication of $2^{-k}$.
\end{proof}

Having established Lemma~\ref{corr-add-reals}, we are ready to prove the result:
\begin{proof}[Proof of Theorem \ref{thm-simulate-mpnn}]
Let $M$ be an MPNN with $k$ layers and width $w$, where the initial states are $\vh_1 \dots \vh_n$. We define an EDU-QGC $\mC$ which computes the same final node embeddings as $M$, based on $M$'s iterated message-passing and node update procedure.

In the following, we conceptually divide the qubits for each node $v$ into $(k+1)\times w$ registers $\vh_v^{(0,0)}, \dots, \vh_v^{(k,w-1)}$ of $b$ qubits each, and $k\times w$ registers $\va_v^{(1,0)}, \dots, \va_v^{(k,w-1)}$ of $b$ qubits each. This is a total of $(k+1)w \times b + kw \times b = (2k+1)wb$ qubits as expected. The $\vs_v^{(0)}$ registers are initialized to the initial MPNN node states $\vh_v$, and all other qubits are set to $\ket{0}$. 

Then, for each MPNN layer, we first simulate its message-passing phase with two-node unitaries for all edges, and afterwards, we simulate the update functions with single-node unitaries.
Specifically, for the $k$-th message-passing layer of $M$, we apply a unitary $U^{(k)}$ for each edge $(v,u)$ that should have the effect of adding the value of $\vh_v^{(k-1,i)}$ to $\va_u^{(k,i)}$ and vice versa for each $i \in \{0 \dots w-1\}$. This results in the $\va_v^{(k, \cdot)}$ registers eventually storing the sum of their neighbors' states from the previous layer, which simulates the sum aggregation. This is an equivariantly diagonalizable unitary acting well on undirected graphs by Lemma $\ref{corr-add-reals}$, so applying it for each edge is a valid EDU-QGC layer.

For the $k$-th update layer, a unitary is applied to each node that XORs the result of the MPNN's update function, $\textsc{update}^{(k)}(\vh_v^{(k-1,\cdot)}, \va_v^{(k)})$ onto the set of registers $\vh_v^{(k,\cdot)}$, which are until this point still initialized to all zeros. This is a permutation and therefore a unitary, so applying it for each node is a valid EDU-QGC layer.

At the end of the circuit, we measure all qubits, which will include the final node states $\vh_v^{(k,\cdot)}$. We can classically aggregate in the same way the MPNN pools its results to give our prediction. This will match the MPNN's output for all inputs with $0$ error probability.
\end{proof}

\subsection{Proof of Theorem~\ref{thm-finite-real-functions}: Universality Result}

We show that EDU-QGCs are universal approximators for (real and Boolean) functions over bounded graph domains, by showing EDU-QGCs can simulate MPNNs extended with random node initialization.

\label{app-universality}
\subsubsection{From Boolean to real-valued functions}
We will prove Theorem \ref{thm-finite-real-functions} by first looking at the case of Boolean-valued functions over graphs, and show that the case for real functions follows by the same argument as \citeauthor{Abboud2021TheSP}~\yrcite{Abboud2021TheSP}. 

\begin{lemma}
\label{thm-finite-bool-functions}
For any Boolean function $f$ defined over $\sG^n$, and any $\epsilon > 0$, there is an EDU-QGC that calculates $f(\gG)$ with probability $(1-\epsilon)$ for any graph $\gG$.
\end{lemma}

Let us start by showing how Theorem \ref{thm-finite-real-functions} follows from Lemma \ref{thm-finite-bool-functions}:

\begin{proof}[Proof of Theorem \ref{thm-finite-real-functions} given Lemma \ref{thm-finite-bool-functions}]
Consider the outputs of any real-valued function $f$ over graphs of size $n$ expressed in binary decimal form, in the form of zeros and ones assigned to different positions. Since there is a finite number of such graphs, there is a finite number $k$ of different decimal places where the result differs for any two graphs. For each of these, a binary classifier can be represented by EDU-QGCs by Lemma \ref{thm-finite-bool-functions} that gives the correct prediction with probability $1-\frac{\epsilon}{k}$. 

Say the $i$-th binary classifier predicts an output $F_i(G) \in \{0, 1\}$ for any bounded-size graph $G$ that represents the bit at position $k_i \in \sZ$ of the desired real-number output. Running these `next to each other' is also a valid EDU-QGC, and their results can then be combined by an MLP to calculate the real output:
\begin{equation*}
    F(G) = \Big(\sum_i F_i(G)\times 2^{k_i}\big)+C
\end{equation*}
By the union bound, the total probability of any classifier making a mistake is $\epsilon$, so with probability $(1-\epsilon)$ our prediction can be as accurate as allowed by our representation of real numbers.
\end{proof}

\subsubsection{Individualizing graphs}

\citeauthor{Abboud2021TheSP}~(\citeyear{Abboud2021TheSP}) prove their results about the power of MPNNs as follows: say a graph is \emph{individualized} if all nodes are extended with unique features. They  construct MPNNs that accurately model any function from a large class assuming the input graph is individualized. And for any graph of $n$ nodes and arbitrarily small desired error rate $\epsilon$, if we randomize some node features appropriately, the result will be individualized with probability at least $(1-\epsilon)$. 

In the case of EDU-QGCs, if we assume some part of all node states is initialized to all $\ket{0}$s, we can have the first EDU-QGC layer apply a unitary on all nodes consisting of Hadamard gates on the appropriate qubits. This maps them to the uniform superposition over all bitstrings. If we then use the construction from Theorem \ref{thm-simulate-mpnn} that acts classically on the computational basis, and then measure the results, we get the same result as running the MPNN with a randomized initial state. The following lemma bounds the number of qubits required for this:

\begin{lemma}
\label{lemma-random-qubits-number}
Putting $n$ sets of $b \geq 2\log(n)+\log(1/\epsilon)$ qubits each in the uniform superposition and measuring them leads to $n$ unique bitstrings with probability at least $(1-\epsilon)$.
\end{lemma}
\begin{proof}
We are effectively just randomizing $b$ classical bits uniformly. If we randomize $b$ individual bits of node state uniformly at random, each pair of nodes would get the same label with just $2^{-b}$ probability. This applies for each of the $n(n-1)/2 < n^2$ pairs of nodes, so by the union bound, the total probability of any match is at most $2^{-b}n^2$. This is less than $\epsilon$ if $b \geq 2\log(n)+\log(1/\epsilon)$ bits are randomized.
\end{proof}

\subsubsection{Achieving universality}
As noted in Section \ref{sec-universality}, we cannot directly rely on the results of either \citeauthor{Abboud2021TheSP}~\yrcite{Abboud2021TheSP} or \citeauthor{Sato2020RandomFS}~\yrcite{Sato2020RandomFS}, and instead give a novel MPNN construction that is partially inspired by \citeauthor{Sato2020RandomFS}, but relies solely on the results of \citeauthor{gin} about their graph isomorphism networks~\yrcite{gin}.

We essentially rely on the following about graph isomorphism networks which follows directly from Corollary 6 of \citeauthor{gin}:
\begin{lemma}
\label{lemma-gin-inj}
Let $\mathcal X$ be a countable set of vectors, and let $\mathcal P_k(\mathcal X)$ be the set of multisets of elements of $\mathcal X$ with size at most $k$. The aggregate-update function of GINs applied to inputs from $(\mathcal X \times \mathcal P_k(\mathcal X))$ (representing a node's previous state and the multiset of its neighbors' previous states) can learn injective functions over such an input space.
\end{lemma}

From this result, we build up to MPNNs that can injectively encode the connected subgraph of each node into their final states if the initial features are unique. To formalize this, we need the following auxiliary definition:

\begin{definition}
For a graph $G$ with initial node features $\vh_v$ for each node $v$, a node $u$ in $G$ and $k \in \sZ^+$, define 
\begin{equation*}
T(G, u, l) = \begin{cases}
\{\!\!\{\vh_u\}\!\!\} &\mbox{if } k = 0 \\
\big(\vh_u, \{\!\!\{T(G, v, k-1) &\mbox{if } k > 0 \\ 
\qquad \qquad |~v \in \mathcal N(u)\}\!\!\}\big) 
\end{cases}
\end{equation*}
where $\mathcal N(u)$ represents the set of neighbors of a node $u$.

Following \citeauthor{Sato2020RandomFS}, we call this a \emph{level-$k$ tree}, and it represents total information propagated to a node in $k$ message-passing steps.
\end{definition}

We show that GINs with $k$ layers can injectively encode the level-$k$ tree of a node:
\begin{lemma}
\label{lemma-injective-gin}
Let $GIN_{\bm{\theta}}(G)_v$ represent the final node features of node $v$ in a graph $G$ after applying a graph isomorphism network with parameters $\bm{\theta}$. There is some configuration $\bm{\theta}^*$ of a $k$-layer GIN such that for any nodes $v_1, v_2$ in degree-bounded graphs $G_1,G_2$ respectively, with initial node features chosen from a countable space, if $T(G_1,v_1,k)\not=T(G_2,v_2,k)$ then $GIN_{\bm{\theta}^*}(G_1)_{v_1} \not= GIN_{\bm{\theta}^*}(G_2)_{v_2}$.
\end{lemma}

\begin{proof}
By induction. The base case $k=1$ follows directly from Lemma \ref{lemma-gin-inj}. The inductive step follows from the same claim, since the outputs of a GIN layer applied to a countable input space still form a countable space: the set of bounded-size multisets from a countable space is still countable, and so is any image of this set under some function.
\end{proof}

Furthermore, we show that the level-$n$ tree of a node in a graph of $n$ nodes identifies the isomorphism class of the node's connected component:
\begin{lemma}
\label{lemma-graph-iso}
Let $G_1,G_2$ be two non-isomorphic graphs with node sets $V_1, V_2$ of size $n$ with node feature vectors $\vh_v$ unique within each graph, and take any $v_1 \in V_1, v_2 \in V_2$. Then, the following statements hold:
\begin{itemize}
    \item If the graphs $G_1$ and $G_2$ are connected, then $T(G_1, v_1, n)\not=T(G_2,v_2,n)$
    \item If the graphs $G_1$ and $G_2$ are not connected, then $T(G_1', v_1, n)\not=T(G_2',v_2,n)$, where $G_1'$ and $G_2'$ are the induced connected components of $v_1$ and $v_2$ in $G_1$ and $G_2$, respectively, representing the isomorphism classes.
   \end{itemize}
\end{lemma}
\begin{proof}
We first prove the case where the graphs are connected. Let $\{\vv_1, \dots, \vv_n\}$ be the unique node feature vectors in $G_1$. Note that all of these will appear in $T(G_1,v_1, n)$, because the features of any nodes at distance $d$ from $v_1$ will appear in $T(G_1, v_1, d)$ by induction, and a connected graph of $n$ nodes has a diameter at most $(n-1)$. Therefore, if $G_2$ contains a different set of unique node features, we get $T(G_1, v_1, n)\not=T(G_2,v_2,n)$ immediately.

Otherwise for each $i$, we can denote $v_i^{(1)}$ as the node in $G_1$ with feature vector $\vv_i$, and $v_i^{(2)}$ as the node in $G_2$ with the same vector. These are unique by the uniqueness of feature vectors. From $T(G_1, v_1, n)$, we can extract the sets $\mathcal N_1(\vv_i) = \{\vh_u~|~ u \in \mathcal N(v_i^{(1)})\}$, i.e., the features of nodes adjacent to the node with the feature vector $\vv_i$. This also follows by induction: $T(G_1, v_1, k)$ recursively includes a tuple $(\vh_u, \{\!\!\{T(G_1,w,k-d-1)~|~w \in \mathcal N(u)\}\!\!\})$ for any $u$ at $d \leq k-1$ steps from $v_1$, and $T(G_1,w,k-d-1)$ gives us $\vh_w$ for any $k,d$. Similarly, from $T(G_2, v_2, n)$, we can extract $\mathcal N_2(\vv_i) =  \{\vh_u~|~ u \in \mathcal N(v_i^{(2)})\}$. If $T(G_1, v_1, n) = T(G_1, v_2, n)$, then $\mathcal N_1(\vv_i) = \mathcal N_2(\vv_i)$ for all $i$, which gives an isomorphism between $G_1$ and $G_2$: the nodes $v_i^{(1)}$ and $v_i^{(2)}$ are in correspondence.

This can be extended to the case for  disconnected graphs because $T(G, v, n) = T(C, v, n)$ for a any graph $G$ with a node $v$ in a connected component $C$, and then the same derivation applies.
\end{proof}

These results finally allow us to prove Lemma \ref{thm-finite-bool-functions} and thereby also complete the proof of Theorem \ref{thm-finite-real-functions}:

\begin{proof}[Proof of Lemma \ref{thm-finite-bool-functions}]
We start by initializing a sufficient number of qubits of each node to $\ket +$ such that with probability $(1-\epsilon)$, observing all $n$ initial node states leads to $n$ unique measurements. By Lemma \ref{lemma-random-qubits-number}, $\lceil 2 \log(n) + \log(1/\epsilon)$ quits suffice. We apply an $n$-layer GIN to this input, which our EDU-QGC can simulate by Theorem \ref{thm-simulate-mpnn}. By combining Lemmas \ref{lemma-injective-gin} and \ref{lemma-graph-iso}, with appropriate parameterization the GIN, the final node states will be an injective function of the node's connected component.

Since there is a finite number of such graphs, the set of the GIN's outputs is bounded, so an MLP applied to the node state can turn this into a vector of indicator variables for each isomorphism class within some required accuracy: let an indicator $I_C^{(v)}$, part of the node state for node $v$, be between $1-\frac{1}{3n}$ and $1$ if the $v$'s component is isomorphic to a graph $C$ (without regard for the random features) and between $0$ and  $\frac{1}{3n}$ otherwise. Since the update function in the GIN architecture is an MLP, this computation can be built into its final layer, which our EDU-QGC can simulate.

We can then pool the node states by summing them into graph-level indicators: for each isomorphism class $C$ of at most $n$-node graphs, the pooled embedding will contain a summed value $N_C$ encoding the number of nodes whose connected component is in that isomorphism class. For each $I_C$, the total error is at most $\frac{1}{3}$, so graphs with a different multiset of connected components will be mapped to different vectors. Since the set of graphs of size $n$ is finite, the space of these vectors is bounded, and we can apply an MLP to these values to learn any Boolean function over bounded graphs. If we construct an MLP with accuracy $0.4$, the output is always more than $0.6$ if the correct answer is $1$ and always less than $0.4$ if the correct answer is $0$. This can be mapped to discrete values in $\{0,1\}$ with perfect accuracy via a continuous function easily representable by further MLP layers. Therefore, the output of the model will be exactly correct as long as observing the $\ket +$ states leads to a unique initial state for each node, which has probability at least $(1-\epsilon)$ as required.
\end{proof}

\section{Analysis of the IQP Experiment}
\label{app-experiment-analysis}

In an effort to better understand the power of such circuits, we focused on analyzing the most well-behaved special case of the above EDU-QGC, with $CZ(\pi)$ rotations. 

Using the ZX-calculus, we show that applying it to any $n$-cycle graph results in a uniform distribution over certain measurement outcomes, give a simple algorithm to check for a given $n$-length bitstring whether it is one of these possible outcomes, and prove that the number of measured $\ket 1$s always has the same parity as the size $n$ of the graph.

With $\alpha=\pi$, the $\alpha$-boxes representing the CZ-gates in Figure \ref{fig-experiment} turn into simple Hadamard. So for any specific bitstring $\ket{b_1\dots b_n}$, we can get the probability of measuring it by simplifying the following scalar:

\ctikzfig{d1}

where the numerical term comes from normalizing each CZ-gate with a factor of $\sqrt{2}$.

We can substitute the appropriate white and gray spiders for the $\ket +, \ket 0$ and $\ket 1$ states to apply ZX-calculus techniques~\cite{pqpbook}: a white spider with phase $0$ for the $\ket +$ state, and gray spiders with $0$ and $\pi$ phases respectively for $\ket 0$ and $\ket 1$. All of these need to be normalized with a factor of $\frac{1}{\sqrt 2}$. Due to the Hadamard gates, these all turn into white spiders that can be fused together, so this is equal to a simple trace calculation:

\ctikzfig{d2}

where $\alpha_i = 0$ if $b_i = 0$ and $\pi$ if $b_i = 1$. 

This can be simplified step by step. Firstly, as long as there are any spiders with $\alpha_i = 0$ and two distinct neighbors (i.e., there are at least 3 nodes in total), we can remove them and fuse their neighbors:
\begin{equation}
  \input{d3.tikz}
\end{equation}

After repeating this, we get one of two outcomes. Firstly, we might end up with one of 3 possible circuits with that still have some $\alpha_i=0$ but less than 3 nodes, which we can evaluate by direct calculation of their matrices:
\begin{equation}
  \input{d5.tikz}
\end{equation}

Or all the remaining spiders have $\alpha_i=\pi$, we can repeatedly eliminate them in groups of 4:
\begin{equation}
  \input{d4.tikz}
\end{equation}

On repeating this, we end up with 0 to 3 nodes with $\alpha_i = \pi$, which we can evaluate directly:
\begin{equation}
  \input{d6.tikz}
\end{equation}

Observe that during the simplifications, we only introduced phases with an absolute value of $1$, which do not affect measurement probabilities. Furthermore, we always decreased the number of nodes involved by 2 or 4, hence the parity is unchanged. This means for odd $n$, we will always end up with one of the odd-cycle base cases with a trace of $0$ or $\pm \sqrt 2$, while for even $n$, we get to the even-cycle base cases with traces of $0$ or $2$. 

Combining with the initial coefficient of $\big(\frac{1}{\sqrt 2}\big)^n$ and taking squared norms, we get that \emph{for odd $n$, each bitstring is observed with probability $0$ or $\frac{1}{2^{n-1}}$} (so half of all possible bitstrings are observed), while \emph{with even $n$, each bitstring is observed with probability $0$ or $\frac{1}{2^{n-2}}$} (so we see only a quarter of all bitstrings).

Furthermore, to check which bitstrings are observed, we can summarize the ZX-diagram simplification as a simple algorithm acting on cyclic bitstrings (where the first and last bits are considered adjacent):

\begin{itemize}
    \item As long as there is a 0 in the bitstring and the length of the bitstring is more than 2, remove the zero along with its two neighbors, and replace them with the XOR of the neighbors.
    \item If you end up with just $\ket{00}$, the state has a positive probability to be observed. If you end up with $\ket{0}$ or $\ket{01}$, it has 0 probability.
    \item When there are only $\ket 1$s remaining, if the number of these is 2 mod 4, the input has 0 probability to be observed, otherwise positive. 
\end{itemize}

This shows us why the observed number of $\ket{1}$s always has the same parity as $n$: at each step, both the parity of $\ket{1}$s and the parity of the bitstring's length is unchanged. The only even-length base case with an odd number of ones is $\ket{01}$, which corresponds to states with $0$ probability; and similarly the only odd-length base case with an even number of ones is $\ket{0}$, which has the same outcome.

We can also derive the specific probabilities observed in the experiment. It's easy to see from this that in the case of a triangle, the observable states are $\ket{001}, \ket{010}, \ket{100}, \ket{111}$. This allows us to calculate the probabilities observed for the case of two triangles. For the 6-cycle, the observable states are $\ket{000000}$, six rotations of $\ket{000101}$, six rotations of $\ket{001111}$, and three rotations of $\ket{101101}$, giving the expected probabilities as well.

\section{The number of parameters of a single-qubit EDU-QGC layer}
\label{app-parameter-count}

In our second experiment, we train EDU-QGC models with alternating node and edge layers. After accounting for redundancy, this setup contains merely 6 real-valued degrees of freedom for each pair of node and edge layers:
\begin{itemize}
    \item The node layer is given by an arbitrary single-qubit unitary, which can be given by 3 Euler-angle rotations of the Bloch sphere.
    \item The edge layer can involve an arbitrary equivariantly diagonalizable unitary $(\mV \otimes \mV)\mD(\mV^\dagger\otimes\mV^\dagger)$ as given in Definition \ref{def-edu}. However, the $\mV$ is redundant when surrounded by two node layers applying single-node unitaries $\mU_1, \mU_2$ everywhere: modifying these to be $\mV \times \mU_1$ and $\mU_2 \times \mV^\dagger$ respectively would have the same effect. Hence it suffices to consider the diagonal unitary $\mD$, which applies some phase in each of the $\ket{00}, \ket{01}, \ket{10}$ and $\ket{11}$ cases. To satisfy the undirected graph constraint of Equation \ref{eq-commutativity-undirected}, the phases for $\ket{01}$ and $\ket{10}$ need to be the same. This leaves us with 3 real parameters for each of the phases.
\end{itemize}

Note that in order to have an efficient implementation, we implemented edge layers as just diagonal unitaries over two nodes. This is justified by the above argument regarding their redundancy for all layers except the last, which is not surrounded by node layers -- in this case it could slightly affect the performance of the model in principle.

\section{Characterising Equivariant Unitaries}
\label{app-eqsc}

While investigating the behavior of EQGCs, we have considered what happens if we restrict $\mC_{\vtheta}(\cdot)$ to only depend on the graph size rather than the adjacency matrix. In this case, for each $n$ it must apply a unitary that treats each node the same. These unitaries are of interest because they could be considered PQCs with an inductive bias for learning functions over \emph{sets} rather than graphs, and they are also the unitaries that any EQGC must assign if given a graph that is either empty or complete. 

\begin{definition}
Let $\EU_s^n$ be the subset of $\C^{s^n \times s^n}$ corresponding to equivariant unitaries over $n$ nodes of dimensionality $s$, i.e. unitaries that satisfy Equation \ref{eq-eqgc-def} in place of $\mC_\vtheta(\cdot)$. 
\end{definition}

These are the matrices that could serve as the value of $\mC_\vtheta(n)$ in an EQGC that did not depend on the adjacency matrix $\mA$. In this appendix, we prove upper and lower bounds on the dimensionality of this set, and show some necessary and some sufficient conditions for an $s^n \times s^n$ matrix to be in $\EU_s^n$. We show that the dimensionaity of $\EU_s^n$ grows without a bound in $n$. This implies that contrary to the closely related \emph{invariant and equivariant networks} studied by \citeauthor{maron2018invariant}~\yrcite{maron2018invariant}, even for our restricted EQGCs, no finite parameterization could achieve all allowed unitaries for arbitrarily high $n$.

We focus on the case $s=2$, but also discuss how one would generalize our results to larger node states.

\subsection{An upper bound: equivariant linear layers}
The unitarity constraint is tricky to analyse, so in this section we will focus on a superset of $\EU_s^n$ with simpler structure:

\begin{definition}
Let $\EU_s^{n,+}$ be the subset of $\C^{s^n \times s^n}$ corresponding to arbitrary complex matrices that satisfy Equation \ref{eq-eqgc-def} in place of $\mC_\vtheta(\cdot)$. 
\end{definition}

First let us consider the case when $s=2$, so each node is assigned a single qubit which is in a superposition of $\ket 0$ and $\ket 1$, and the action of any matrix $\mL$ in $\EU_2^{n,+}$ can be represented as mapping bitstrings of length $n$ (i.e., computational basis vectors in $\C^{2^n}$) to linear combinations of such bitstrings. The general case for $s > 2$ is conceptually analogous, but this is easier to state and prove clearly.

\begin{theorem}
\label{thm-ELL}
A matrix $\mL \in \C^{2^n \times 2^n}$ is in $\EU_2^{n,+}$ if and only if it can be expressed by weights $w_{ijk} \in \C$ for $0 \leq i \leq n$, $0 \leq j \leq i$ and $0 \leq k \leq n-i$ as follows: for computational basis states $\ket \psi, \ket \theta$, $\bra \theta \mL \ket \psi = w_{ijk}$ if the bitstring representing $\ket \psi$ contains $\ket 1$s in $i$ different positions, the bitstring representing $\ket \theta$ contains $j$ $\ket 1$s at positions where $\ket \psi$ had $\ket 1$s and $k$ $\ket 1$s at positions where $\ket \psi$ had $\ket 0$s.
\end{theorem}

\begin{example}
For $n=3$,
\begin{equation}
    \begin{split}
        \mL\ket{100} = &w_{100}\ket{000} + \\ &w_{101}\big(\ket{001}+\ket{010}\big) + \\ &w_{102}\ket{011} +\\ &w_{110}\ket{100}+\\&w_{111}\big(\ket{101}+\ket{110}\big) +\\& w_{112}\ket{111}
    \end{split}
\end{equation}
This shows that $\bra{001}\mL\ket{100}=\bra{010}\mL\ket{100}=w_{101}$ since $\bra{001}$ and $\bra{010}$ both contain one $\bra{1}$ in a position where where $\ket{100}$ has a $\ket{0}$, and no $\bra{1}$ where $\ket{100}$ has a $\ket{1}$; and $\bra{101}\mL\ket{100}=\bra{110}\mL\ket{100}=w_{111}$, because they both contain one $\bra{1}$ for the $\ket{0}$s in $\ket{100}$ and one $\bra{1}$ for the single $\ket{1}$ in $\ket{100}$. The other inner products involving $\ket{100}$ all differ in how many $\bra{1}$s meet $\ket{1}$s and $\ket{0}$s, so they can be chosen independently of each other. (Note however that they are not independent of other values of the matrix $\mL$ such as those in the vectors $\mL\ket{001}$ and $\mL\ket{010}$, as we will see in the proof.)
\end{example}
For further clarity, consider representing the following EQSCs with such weights:

\begin{example}[$CZ(\alpha)$-gates between all pairs of nodes]
Consider a circuit $\mL$ consisting of controlled Z-rotations with a parameter $\alpha$ applied between each pair of qubits. For computational basis states $\ket{e_1}, \ket{e_2}$, this only applies phases, therefore we have a diagonal matrix and $\bra{e_1}\mL\ket{e_2} = 0$ if $\ket{e_1} \not= \ket{e_2}$. The phase applied is $e^{-i\alpha}$ for each pair of qubits that are both set to one, so if the input contains $i$ ones then we get a phase of $e^{-i(i-1)\alpha/2}$ in total. Therefore $\mL$ is represented by $w_{ijk} = e^{-i(i-1)\alpha/2}$ if $j=i, k=0$ and $0$ otherwise.

\end{example}

\begin{example}[Arbitrary single-qubit unitaries applied everywhere]

Let $\mU = \big(\begin{smallmatrix}
  u_{0,0} & u_{0,1}\\
  u_{1,0} & u_{1,1}
\end{smallmatrix}\big)$. Then, for $x, y \in \{0, 1\}$, we have $\bra{x}\mU\ket{y} = u_{x,y}$. Suppose we apply this unitary to all $n$ qubits. Then, for two computational basis states $\ket{e_1}, \ket{e_2}$, $\bra{e_1}\mU^{\otimes n}\ket{e_2}$ is of the form $u_{0,0}^a\times u_{0,1}^b\times u_{1,0}^c\times u_{1,1}^d$, where $a$ and $d$ are the number of overlapping $\ket0$s and $\ket1$s respectively in the bitstring representation of $\ket{e_1}, \ket{e_2}$, $b$ is the number of positions where $\ket{e_1}$ contains a $\ket{0}$ and $\ket{e_2}$ contains a $\ket{1}$, and $c$ is the same in the other direction.

This lets us express the $w_{ijk}$ parameters representing $\mU^{\otimes n}$ from inner products of computational basis states $\bra{e_1}\mU\ket{e_2}$ and expressing $a,b,c,d$ as above: 
\begin{itemize}
    \item $d$, the number of overlapping ones, is just $j$.
    \item $c$, the number of ones in $\bra{e_1}$ meeting zeros in $\ket{e_2}$, is just $k$.
    \item We can get $b$, the number of zeros in $\bra{e_1}$ meeting ones in $\ket{e_2}$ as $i-j$, subtracting the overlapping ones from the number of ones in the input.
    \item We can get $a$, the number of overlapping zeros as $(n-i)-k$, getting the number of zeros in $\ket{e_2}$ as $(n-i)$ and then subtracting the $k$ positions where $\bra{e_1}$ has a one.
\end{itemize}
So we get that $\mU^{\otimes n}$ is represented by $w_{ijk} = u_{0,0}^{n-i-k}\times u_{0,1}^{i-j}\times u_{1,0}^{k}\times u_{1,1}^{j}$.

\end{example}

We will prove this theorem through two simple lemmas.

\begin{lemma}
\label{lemma-fixed-eqgc-1}
Any matrix $\mL \in \EU_2^{n,+}$ is entirely characterized by its output on $\ket {s_0} = \ket{00\dots00}, \ket{s_1} = \ket{00\dots 01}, \dots, \ket{s_{n-1}}=\ket{01\dots 11}, \ket{s_n} = \ket{11\dots11}$.
\end{lemma}

\begin{proof}
Consider any the computational basis vector $\ket e \in \C^{2^n}$. This corresponds to some string of zeros and ones. Then, for the $\ket {s_i}$ containing the same number of zeros and ones, there is some permutation of indices $\Tilde \mP \in \C^{2^n \times 2^n}$ such that $\ket e = \Tilde\mP\ket {s_i}$ and therefore $\mL\ket e = \mL\Tilde\mP\ket{s_i}$. Multiplying by $\Tilde \mP^T$ gives $\Tilde\mP^T\mL\ket e = \Tilde\mP^T \mL\Tilde\mP\ket{s_i} = \mL\ket{s_i}$ by equivariance, so $\mL\ket e = \Tilde\mP\mL\ket{s_i}$. So knowing $\mL\ket{s_i}$ for each $\ket{s_i}$ determines $\mL\ket{e}$ for all computational basis vector, hence determining it entirely.
\end{proof}

\begin{lemma}
\label{lemma-fixed-eqgc-2}
We must have $\bra{e_1}\mL\ket{s_i} = \bra{e_2}\mL\ket{s_i}$ for computational basis vectors $\ket{e_1}, \ket{e_2}$ which can be transformed to each other by permuting over indices that have the same value (0 or 1) in $\ket{s_i}$.
\end{lemma}

\begin{proof}
Consider the permutation of indices $\Tilde \mP$ that turns $\ket{e_1}$ to $\ket{e_2}$. Note that $\Tilde\mP\ket{s_i} = \ket{s_i}$ by the given premise, so by equivariance we have $\bra{e_1}\mL\ket{s_i} = \bra{e_1}\Tilde\mP^T\mL\Tilde\mP\ket{s_i} = \bra{e_1}\Tilde\mP^T\mL\ket{s_i} = \bra{e_2}\mL\ket{s_i}$.
\end{proof}

\begin{proof}[Proof of Theorem \ref{thm-ELL}]
Lemma \ref{lemma-fixed-eqgc-2} showed that $\mL\ket{s_i}$ expressed in the computational basis will have the same weight for any basis vector with $j$ ones where $\ket{s_i}$ had ones, and $k$ ones where $\ket{s_i}$ had zeros. Denote this weight $w_{ijk}$. By Lemma \ref{lemma-fixed-eqgc-1}, these parameters uniquely characterize the equivariant linear layer.

This proves the theorem in the forward direction: any matrix in $\EU_2^{n,+}$ can be characterized by weights $w_{ijk}$. Now we show the other direction, that any linear transformation characterized by an arbitrary choice of $w_{ijk}$ satisfies Equation \ref{eq-eqgc-def} and therefore is in $\EU_2^{n,+}$. Consider an arbitrary $\mL \in \C^{2^n \times 2^n}$ given in this form. It suffices to show that it behaves correctly with respect to swap permutations and input states in the computational basis: more complex permutations can be built by composing swaps, and more complex states by linear combinations of basis states. For any bitstring input $\ket e$, we can have two kinds of swaps:

\begin{itemize}
    \item In the first case, we swap two indices with the same digit in the bitstring (both $\ket 0$ or $\ket 1$). The input to $\mL$ is unchanged, and equivariance is respected because the same coefficients from $w_{ijk}$ are multiplying pairs of output vectors that should be swapped.
    \item In the second case, the digits at the two indices differ. The inputs passed to $\mL$ on the two sides of the equation are different, and equivariance is ensured by the number of overlapping $\ket 1$s changing in a way that the $w_{ijk}$ coefficients get swapped consistently.
\end{itemize}
\end{proof}

As a consequence, we can easily see that the dimensionality of the set $\EU_2^{n,+}$ is unbounded in terms of $n$, as opposed to the equivariant layers studied by \citeauthor{maron2018invariant}~\yrcite{maron2018invariant}, so we cannot hope to uniformly parameterize the entire space for unbounded $n$.

\begin{corollary}
The dimensionality of the set $\EU_2^{n,+}$ with a single qubit per node over $n$ nodes is:

$$\sum_{i=0}^n (i+1)(n-i+1) = \frac{1}{6}n(n+1)(n+5)$$
\end{corollary}
\label{cor-eqsc-upperbound}
\begin{proof}
The left-hand side follows from the above by considering the number of $(i,j,k)$ triples with $0 \leq i \leq n$, $0 \leq j \leq i$, $0 \leq k \leq n-i$. We get the closed form on the right using the formula for pyramid numbers and simplifying.
\end{proof}

\subsubsection{Generalizing to larger node states}

An analogous result holds for $\EU_s^{n,+}$ with $s > 2$. Say we have $s$ possible node basis states $\{\ket 0, \dots, \ket{s-1}\}$. In this case, a single matrix element $\bra \theta \mL \ket \psi$ for computational basis states $\ket \psi, \ket \theta$ is depends on the entire set of how many $\ket{i}$ appear in $\ket \psi$ in positions where $\ket \theta$ contains a $\ket j$, for all $i,j\in \{0\dots s-1\}$. 

To prove this, similarly to Lemma $\ref{lemma-fixed-eqgc-1}$, we can show that it suffices to specify $\mL\ket \psi$ for each distribution of input node states; and similarly to Lemma \ref{lemma-fixed-eqgc-2}, we can show that $\bra \theta\mL\ket \psi$ is invariant to changing $\bra \theta$ in a way that does not change the number of any $\bra{i}$ to $\ket{j}$ `matches' as described above.

\subsubsection{Implications for $\EU_s^{n}$}
Corollary \ref{cor-eqsc-upperbound} has implications for our original set of equivariant unitaries $\EU_s^{n}$ -- it gives an upper bound for the dimensionality of the set. 

\subsection{A lower bound: diagonal equivariant unitaries}
\label{sec-dimension-lowerbound}

To see whether the size of the space of EQSCs grows with the size of the set, we can investigate a more restricted space as a lower bound: \emph{diagonal} unitaries satisfying the equivariance condition in Equation \ref{eq-eqgc-def}. 

A general diagonal unitary can apply an arbitrary phase to each computational basis state independently. The equivariance condition restricts us to applying the same phase for inputs that could be transformed to each other by permuting the indices, i.e. inputs that contain the same distribution of node states (the same number of $\ket{0}$s and $\ket{1}$s when using one qubit per node). This gives a lower bound of $n+1$ on the dimensionality of equivariant unitaries over sets of size $n$ using a single qubit per node, which is still unbounded in $n$. More generally, for $n$ nodes with $s$ possible states each, the lower bound is the number of unique $s$-tuples of nonnegative integers that sum to $n$, which is given by ${n+s-1 \choose s-1}$. This is also a lower bound on the dimensionality of $\EU_s^n$. 

\subsection{Comparison with classical invariant/equivariant graph networks}
In their paper on invariant and equivariant graph networks, \citeauthor{maron2018invariant} ask similar questions to characterize and implement classical equivariant/invariant models operating on tensors representing relational data~\yrcite{maron2018invariant}. While the questions we investigate were partly inspired by them, and our data can also be seen as high-order tensors, there are significant differences in our setting.

Most importantly, the order of $k$ the tensors they dealt with was fixed and independent of the size $n$ of the input graph, while the size of the tensors along each of those $k$ dimensions depended $n$. For example, their input included the adjacency matrix, a tensor in $\R^{n^2}$. For EQGCs, this is the other way around. Adding more nodes means working with a larger tensor product, but each dimension is of a fixed size $s$. For example, with a single qubit per node, our state is in $\C^{2^n}$. This matters for the notion of equivariance/invariance: applying a permutation $p$ brings the element at an index $(i_1, i_2, \dots, i_n)$ to $(i_{p(1)}, i_{p(2)}, \dots, i_{p(n)})$ for us, instead of $(p(i_1), p(i_2), \dots, p(i_n))$ as in the previous work.

Finally, there are a few more obvious differences: due to the quantum context, we are working with complex numbers rather than reals, and we are interested in the extra condition of unitarity rather than arbitrary linear layers.

\end{document}

%% file: math_commands.tex
\usepackage{amsmath,amsfonts,bm}

\def\eqref#1{equation~\ref{#1}}

\def\1{\bm{1}}

\def\eps{{\epsilon}}

\def\vtheta{{\bm{\theta}}}
\def\va{{\bm{a}}}

\def\vh{{\bm{h}}}

\def\vs{{\bm{s}}}

\def\vv{{\bm{v}}}

\def\vx{{\bm{x}}}
\def\vy{{\bm{y}}}

\def\mA{{\bm{A}}}
\def\mB{{\bm{B}}}
\def\mC{{\bm{C}}}
\def\mD{{\bm{D}}}

\def\mH{{\bm{H}}}
\def\mI{{\bm{I}}}

\def\mL{{\bm{L}}}

\def\mP{{\bm{P}}}

\def\mR{{\bm{R}}}
\def\mS{{\bm{S}}}

\def\mU{{\bm{U}}}
\def\mV{{\bm{V}}}

\DeclareMathAlphabet{\mathsfit}{\encodingdefault}{\sfdefault}{m}{sl}
\SetMathAlphabet{\mathsfit}{bold}{\encodingdefault}{\sfdefault}{bx}{n}

\def\gG{{\mathcal{G}}}

\def\sB{{\mathbb{B}}}

\def\sG{{\mathbb{G}}}

\def\sZ{{\mathbb{Z}}}

\newcommand{\R}{\mathbb{R}}
\newcommand{\C}{\mathbb{C}}



%% file: test.tikzstyles
\tikzstyle{right label}=[label, anchor=west, xshift=-1.5mm]
\tikzstyle{inline text}=[text height=1.5ex, text depth=0.25ex, yshift=0.5mm]
\tikzstyle{dot}=[inner sep=0mm, minimum width=3.5mm, minimum height=3.5mm, draw, shape=circle]
\tikzstyle{black dot}=[dot, fill=black]
\tikzstyle{white dot}=[dot, fill=white, text depth=-0.5mm]
\tikzstyle{out label white dot}=[dot, right label, fill=white, text depth=-0.5mm]
\tikzstyle{gray dot}=[dot, fill={gray!40!white}, text depth=-0.5mm]
\tikzstyle{hadamard}=[fill=white, draw, inner sep=0.6mm, font={\footnotesize}, minimum height=4mm, minimum width=4mm]
\tikzstyle{small box}=[rectangle, inline text, fill=white, draw, minimum height=5mm, yshift=-0.5mm, minimum width=5mm, font={\small}]
\tikzstyle{medium box}=[rectangle, inline text, fill=white, draw, minimum height=5mm, yshift=-0.5mm, minimum width=10mm, font={\small}]
\tikzstyle{semilarge box}=[rectangle, inline text, fill=white, draw, minimum height=5mm, yshift=-0.5mm, minimum width=12.5mm, font={\small}]
\tikzstyle{large box}=[rectangle, inline text, fill=white, draw, minimum height=5mm, yshift=-0.5mm, minimum width=15mm, font={\small}]
\tikzstyle{effect}=[regular polygon, regular polygon sides=3, draw, scale=0.75, inner sep=-0.5pt, minimum width=9mm, fill=white, regular polygon rotate=180]
\tikzstyle{state}=[regular polygon, regular polygon sides=3, draw, scale=0.75, inner sep=-0.5pt, minimum width=9mm, fill=white]
\tikzstyle{langstate}=[fill=white, draw, shape=isosceles triangle, shape border rotate=90, isosceles triangle stretches=true, inner sep=0pt, minimum width=1.5cm, minimum height=6.12mm]

\tikzstyle{new edge style 0}=[-]
\tikzstyle{dashed edge}=[-, dashed]
\tikzstyle{thick}=[-, very thick]

%% file: figures/commutativity.tikz
\begin{tikzpicture}
	\begin{pgfonlayer}{nodelayer}
		\node [style=large box] (0) at (-4.5, -1) {$U$};
		\node [style=none] (1) at (-5.5, -0.5) {};
		\node [style=none] (2) at (-3.5, -0.5) {};
		\node [style=none] (3) at (-3.5, 0.5) {};
		\node [style=large box] (5) at (-2.5, 1) {$U$};
		\node [style=none] (6) at (-3.5, 1.5) {};
		\node [style=none] (7) at (-3.5, 2.5) {};
		\node [style=none] (8) at (-1.5, 1.5) {};
		\node [style=none] (9) at (-1.5, 2.5) {};
		\node [style=none] (10) at (-1.5, 0.5) {};
		\node [style=none] (11) at (-3.5, -1.5) {};
		\node [style=none] (12) at (-3.5, -2.5) {};
		\node [style=none] (13) at (-5.5, -2.5) {};
		\node [style=none] (14) at (-5.5, -1.5) {};
		\node [style=none] (15) at (-1.5, -2.5) {};
		\node [style=none] (16) at (-5.5, 2.5) {};
		\node [style=none] (17) at (0, 0) {$=$};
		\node [style=large box] (18) at (2.5, 1) {$U$};
		\node [style=none] (19) at (1.5, -0.5) {};
		\node [style=none] (20) at (3.5, -0.5) {};
		\node [style=none] (21) at (3.5, 0.5) {};
		\node [style=large box] (22) at (4.5, -1) {$U$};
		\node [style=none] (23) at (3.5, 1.5) {};
		\node [style=none] (24) at (3.5, 2.5) {};
		\node [style=none] (25) at (5.5, 1.5) {};
		\node [style=none] (26) at (5.5, 2.5) {};
		\node [style=none] (27) at (5.5, 0.5) {};
		\node [style=none] (28) at (3.5, -1.5) {};
		\node [style=none] (29) at (3.5, -2.5) {};
		\node [style=none] (30) at (1.5, -2.5) {};
		\node [style=none] (31) at (1.5, -1.5) {};
		\node [style=none] (32) at (5.5, -2.5) {};
		\node [style=none] (33) at (1.5, 2.5) {};
	\end{pgfonlayer}
	\begin{pgfonlayer}{edgelayer}
		\draw (3.center) to (2.center);
		\draw (10.center) to (15.center);
		\draw (1.center) to (16.center);
		\draw (7.center) to (6.center);
		\draw (8.center) to (9.center);
		\draw (11.center) to (12.center);
		\draw (14.center) to (13.center);
		\draw (21.center) to (20.center);
		\draw (27.center) to (32.center);
		\draw (19.center) to (33.center);
		\draw (24.center) to (23.center);
		\draw (25.center) to (26.center);
		\draw (28.center) to (29.center);
		\draw (31.center) to (30.center);
		\draw (31.center) to (19.center);
		\draw (20.center) to (28.center);
		\draw (23.center) to (21.center);
		\draw (25.center) to (27.center);
	\end{pgfonlayer}
\end{tikzpicture}

%% file: figures/undirected-unitary.tikz
\begin{tikzpicture}
	\begin{pgfonlayer}{nodelayer}
		\node [style=large box] (0) at (-2.5, 0) {$U$};
		\node [style=none] (5) at (-1.5, -0.5) {};
		\node [style=none] (6) at (-1.5, -1.5) {};
		\node [style=none] (7) at (-3.5, -1.5) {};
		\node [style=none] (8) at (-3.5, -0.5) {};
		\node [style=none] (9) at (-1.5, 1.5) {};
		\node [style=none] (10) at (-1.5, 0.5) {};
		\node [style=none] (11) at (-3.5, 0.5) {};
		\node [style=none] (12) at (-3.5, 1.5) {};
		\node [style=none] (15) at (3.5, -0.5) {};
		\node [style=none] (16) at (3.5, -1.5) {};
		\node [style=none] (17) at (1.5, -0.5) {};
		\node [style=none] (18) at (1.5, -1.5) {};
		\node [style=none] (23) at (3.5, 0.5) {};
		\node [style=none] (24) at (3.5, 1.5) {};
		\node [style=none] (25) at (1.5, 1.5) {};
		\node [style=none] (26) at (1.5, 0.5) {};
		\node [style=large box] (27) at (2.5, 0) {$U$};
		\node [style=none] (28) at (0, 0) {$=$};
	\end{pgfonlayer}
	\begin{pgfonlayer}{edgelayer}
		\draw [in=-90, out=90, looseness=0.50] (6.center) to (8.center);
		\draw [in=-90, out=90, looseness=0.50] (7.center) to (5.center);
		\draw [in=-90, out=90, looseness=0.50] (10.center) to (12.center);
		\draw [in=-90, out=90, looseness=0.50] (11.center) to (9.center);
		\draw (16.center) to (15.center);
		\draw (26.center) to (18.center);
		\draw (23.center) to (24.center);
		\draw (25.center) to (26.center);
	\end{pgfonlayer}
\end{tikzpicture}

%% file: figures/edu-qgc.tikz
\begin{tikzpicture}
	\begin{pgfonlayer}{nodelayer}
		\node [style=large box] (0) at (-2, 7.5) {$D$};
		\node [style=small box] (5) at (-3, 9) {$V$};
		\node [style=none] (7) at (-3, 10) {};
		\node [style=none] (9) at (-1, 10) {};
		\node [style=none] (16) at (-5, 10) {};
		\node [style=none] (17) at (0.5, 5.25) {$=$};
		\node [style=none] (39) at (-1, 0.5) {};
		\node [style=none] (40) at (-3, 0.5) {};
		\node [style=none] (41) at (-5, 0.5) {};
		\node [style=small box] (42) at (-1, 9) {$V$};
		\node [style=small box] (43) at (-1, 6) {$V^\dagger$};
		\node [style=small box] (44) at (-3, 6) {$V^\dagger$};
		\node [style=small box] (45) at (-3, 4.5) {$V$};
		\node [style=small box] (46) at (-5, 4.5) {$V$};
		\node [style=small box] (47) at (-3, 1.5) {$V^\dagger$};
		\node [style=small box] (48) at (-5, 1.5) {$V^\dagger$};
		\node [style=large box] (49) at (-4, 3) {$D$};
		\node [style=large box] (50) at (5, 6) {$D$};
		\node [style=small box] (51) at (4, 7.5) {$V$};
		\node [style=none] (52) at (4, 8.5) {};
		\node [style=none] (53) at (6, 8.5) {};
		\node [style=none] (54) at (2, 8.5) {};
		\node [style=none] (55) at (6, 2) {};
		\node [style=none] (56) at (4, 2) {};
		\node [style=none] (57) at (2, 2) {};
		\node [style=small box] (58) at (6, 7.5) {$V$};
		\node [style=small box] (59) at (6, 3) {$V^\dagger$};
		\node [style=small box] (62) at (2, 7.5) {$V$};
		\node [style=small box] (63) at (4, 3) {$V^\dagger$};
		\node [style=small box] (64) at (2, 3) {$V^\dagger$};
		\node [style=large box] (65) at (3, 4.5) {$D$};
		\node [style=none] (66) at (-6.5, -5.25) {$=$};
		\node [style=large box] (67) at (-2, -6) {$D$};
		\node [style=small box] (68) at (-3, -3) {$V$};
		\node [style=none] (69) at (-3, -2) {};
		\node [style=none] (70) at (-1, -2) {};
		\node [style=none] (71) at (-5, -2) {};
		\node [style=none] (72) at (-1, -8.5) {};
		\node [style=none] (73) at (-3, -8.5) {};
		\node [style=none] (74) at (-5, -8.5) {};
		\node [style=small box] (75) at (-1, -3) {$V$};
		\node [style=small box] (76) at (-1, -7.5) {$V^\dagger$};
		\node [style=small box] (77) at (-5, -3) {$V$};
		\node [style=small box] (78) at (-3, -7.5) {$V^\dagger$};
		\node [style=small box] (79) at (-5, -7.5) {$V^\dagger$};
		\node [style=large box] (80) at (-4, -4.5) {$D$};
		\node [style=large box] (81) at (5, -7.5) {$D$};
		\node [style=small box] (82) at (4, -1.5) {$V$};
		\node [style=none] (83) at (4, -0.5) {};
		\node [style=none] (84) at (6, -0.5) {};
		\node [style=none] (85) at (2, -0.5) {};
		\node [style=none] (86) at (6, -10) {};
		\node [style=none] (87) at (4, -10) {};
		\node [style=none] (88) at (2, -10) {};
		\node [style=small box] (89) at (2, -1.5) {$V$};
		\node [style=small box] (90) at (2, -4.5) {$V^\dagger$};
		\node [style=small box] (91) at (4, -4.5) {$V^\dagger$};
		\node [style=small box] (92) at (4, -6) {$V$};
		\node [style=small box] (93) at (6, -6) {$V$};
		\node [style=small box] (94) at (4, -9) {$V^\dagger$};
		\node [style=small box] (95) at (6, -9) {$V^\dagger$};
		\node [style=large box] (96) at (3, -3) {$D$};
		\node [style=none] (97) at (0.5, -5.25) {$=$};
	\end{pgfonlayer}
	\begin{pgfonlayer}{edgelayer}
		\draw (9.center) to (39.center);
		\draw (40.center) to (7.center);
		\draw (16.center) to (41.center);
		\draw (53.center) to (55.center);
		\draw (56.center) to (52.center);
		\draw (54.center) to (57.center);
		\draw (70.center) to (72.center);
		\draw (73.center) to (69.center);
		\draw (71.center) to (74.center);
		\draw (84.center) to (86.center);
		\draw (87.center) to (83.center);
		\draw (85.center) to (88.center);
	\end{pgfonlayer}
\end{tikzpicture}

%% file: figures/circuit1.tikz
\begin{tikzpicture}
	\begin{pgfonlayer}{nodelayer}
		\node [style=state] (0) at (-10, 2.5) {$+$};
		\node [style=state] (1) at (-6, 2.5) {$+$};
		\node [style=state] (3) at (-2, 2.5) {$+$};
		\node [style=dot] (4) at (-10, 0.5) {};
		\node [style=hadamard] (5) at (-8, 0.5) {$\alpha$};
		\node [style=dot] (6) at (-6, 0.5) {};
		\node [style=white dot] (8) at (-2, 0.5) {};
		\node [style=hadamard] (10) at (-10, -1.5) {};
		\node [style=hadamard] (11) at (-6, -1.5) {};
		\node [style=hadamard] (13) at (-2, -1.5) {};
		\node [style=none] (14) at (-10, -2.5) {};
		\node [style=none] (15) at (-6, -2.5) {};
		\node [style=none] (17) at (-2, -2.5) {};
		\node [style=hadamard] (19) at (-4.25, -0.75) {$\alpha$};
		\node [style=hadamard] (20) at (-4, 0.5) {$\alpha$};
		\node [style=state] (21) at (2, 2.5) {$+$};
		\node [style=state] (22) at (6, 2.5) {$+$};
		\node [style=state] (23) at (10, 2.5) {$+$};
		\node [style=dot] (24) at (2, 0.5) {};
		\node [style=hadamard] (25) at (4, 0.5) {$\alpha$};
		\node [style=dot] (26) at (6, 0.5) {};
		\node [style=white dot] (27) at (10, 0.5) {};
		\node [style=hadamard] (28) at (2, -1.5) {};
		\node [style=hadamard] (29) at (6, -1.5) {};
		\node [style=hadamard] (30) at (10, -1.5) {};
		\node [style=none] (31) at (2, -2.5) {};
		\node [style=none] (32) at (6, -2.5) {};
		\node [style=none] (33) at (10, -2.5) {};
		\node [style=hadamard] (34) at (7.75, -0.75) {$\alpha$};
		\node [style=hadamard] (35) at (8, 0.5) {$\alpha$};
	\end{pgfonlayer}
	\begin{pgfonlayer}{edgelayer}
		\draw [style=new edge style 0] (0) to (4);
		\draw [style=new edge style 0] (4) to (5);
		\draw [style=new edge style 0] (5) to (6);
		\draw [style=new edge style 0] (6) to (1);
		\draw [style=new edge style 0] (8) to (3);
		\draw [style=new edge style 0] (8) to (13);
		\draw (6) to (11);
		\draw (4) to (10);
		\draw [style=new edge style 0] (10) to (14.center);
		\draw [style=new edge style 0] (11) to (15.center);
		\draw [style=new edge style 0] (13) to (17.center);
		\draw [style=new edge style 0, in=180, out=-30, looseness=0.75] (4) to (19);
		\draw [style=new edge style 0, in=-135, out=0] (19) to (8);
		\draw (6) to (20);
		\draw (20) to (8);
		\draw [style=new edge style 0] (21) to (24);
		\draw [style=new edge style 0] (24) to (25);
		\draw [style=new edge style 0] (25) to (26);
		\draw [style=new edge style 0] (26) to (22);
		\draw [style=new edge style 0] (27) to (23);
		\draw [style=new edge style 0] (27) to (30);
		\draw (26) to (29);
		\draw (24) to (28);
		\draw [style=new edge style 0] (28) to (31.center);
		\draw [style=new edge style 0] (29) to (32.center);
		\draw [style=new edge style 0] (30) to (33.center);
		\draw [style=new edge style 0, in=180, out=-30, looseness=0.75] (24) to (34);
		\draw [style=new edge style 0, in=-135, out=0] (34) to (27);
		\draw (26) to (35);
		\draw (35) to (27);
	\end{pgfonlayer}
\end{tikzpicture}

%% file: figures/circuit2.tikz
\begin{tikzpicture}
	\begin{pgfonlayer}{nodelayer}
		\node [style=state] (36) at (-10, 2.5) {$+$};
		\node [style=state] (37) at (-6, 2.5) {$+$};
		\node [style=state] (38) at (-2, 2.5) {$+$};
		\node [style=dot] (39) at (-10, 0.5) {};
		\node [style=hadamard] (40) at (-8, 0.5) {$\alpha$};
		\node [style=dot] (41) at (-6, 0.5) {};
		\node [style=white dot] (42) at (-2, 0.5) {};
		\node [style=hadamard] (43) at (-10, -1.5) {};
		\node [style=hadamard] (44) at (-6, -1.5) {};
		\node [style=hadamard] (45) at (-2, -1.5) {};
		\node [style=none] (46) at (-10, -2.5) {};
		\node [style=none] (47) at (-6, -2.5) {};
		\node [style=none] (48) at (-2, -2.5) {};
		\node [style=hadamard] (50) at (-4, 0.5) {$\alpha$};
		\node [style=state] (51) at (2, 2.5) {$+$};
		\node [style=state] (52) at (6, 2.5) {$+$};
		\node [style=state] (53) at (10, 2.5) {$+$};
		\node [style=dot] (54) at (2, 0.5) {};
		\node [style=hadamard] (55) at (4, 0.5) {$\alpha$};
		\node [style=dot] (56) at (6, 0.5) {};
		\node [style=white dot] (57) at (10, 0.5) {};
		\node [style=hadamard] (58) at (2, -1.5) {};
		\node [style=hadamard] (59) at (6, -1.5) {};
		\node [style=hadamard] (60) at (10, -1.5) {};
		\node [style=none] (61) at (2, -2.5) {};
		\node [style=none] (62) at (6, -2.5) {};
		\node [style=none] (63) at (10, -2.5) {};
		\node [style=hadamard] (64) at (0, -0.75) {$\alpha$};
		\node [style=hadamard] (65) at (8, 0.5) {$\alpha$};
		\node [style=hadamard] (68) at (0, 0.5) {$\alpha$};
	\end{pgfonlayer}
	\begin{pgfonlayer}{edgelayer}
		\draw [style=new edge style 0] (36) to (39);
		\draw [style=new edge style 0] (39) to (40);
		\draw [style=new edge style 0] (40) to (41);
		\draw [style=new edge style 0] (41) to (37);
		\draw [style=new edge style 0] (42) to (38);
		\draw [style=new edge style 0] (42) to (45);
		\draw (41) to (44);
		\draw (39) to (43);
		\draw [style=new edge style 0] (43) to (46.center);
		\draw [style=new edge style 0] (44) to (47.center);
		\draw [style=new edge style 0] (45) to (48.center);
		\draw (41) to (50);
		\draw (50) to (42);
		\draw [style=new edge style 0] (51) to (54);
		\draw [style=new edge style 0] (54) to (55);
		\draw [style=new edge style 0] (55) to (56);
		\draw [style=new edge style 0] (56) to (52);
		\draw [style=new edge style 0] (57) to (53);
		\draw [style=new edge style 0] (57) to (60);
		\draw (56) to (59);
		\draw (54) to (58);
		\draw [style=new edge style 0] (58) to (61.center);
		\draw [style=new edge style 0] (59) to (62.center);
		\draw [style=new edge style 0] (60) to (63.center);
		\draw [style=new edge style 0, in=-150, out=0, looseness=0.75] (64) to (57);
		\draw (56) to (65);
		\draw (65) to (57);
		\draw [in=180, out=-15] (39) to (64);
		\draw (68) to (54);
		\draw (68) to (42);
	\end{pgfonlayer}
\end{tikzpicture}

%% file: figures/commutativity-directed-1.tikz
\begin{tikzpicture}
	\begin{pgfonlayer}{nodelayer}
		\node [style=none] (0) at (0, 0) {$=$};
		\node [style=large box] (1) at (2.5, 1) {$U$};
		\node [style=none] (2) at (1.5, -0.5) {};
		\node [style=large box] (3) at (4.5, -1) {$U$};
		\node [style=none] (4) at (3.5, 1.5) {};
		\node [style=none] (5) at (3.5, 2.5) {};
		\node [style=none] (6) at (5.5, 1.5) {};
		\node [style=none] (7) at (5.5, 2.5) {};
		\node [style=none] (8) at (3.5, -1.5) {};
		\node [style=none] (9) at (3.5, -2.5) {};
		\node [style=none] (10) at (1.5, -2.5) {};
		\node [style=none] (11) at (1.5, -1.5) {};
		\node [style=none] (12) at (5.5, -2.5) {};
		\node [style=none] (13) at (1.5, 2.5) {};
		\node [style=none] (14) at (1.5, 1.5) {};
		\node [style=none] (15) at (1.5, 0.5) {};
		\node [style=none] (16) at (5.5, -1.5) {};
		\node [style=large box] (17) at (-4.5, -1) {$U$};
		\node [style=none] (18) at (-5.5, -0.5) {};
		\node [style=large box] (19) at (-2.5, 1) {$U$};
		\node [style=none] (20) at (-3.5, -1.5) {};
		\node [style=none] (21) at (-3.5, -2.5) {};
		\node [style=none] (22) at (-5.5, -2.5) {};
		\node [style=none] (23) at (-5.5, -1.5) {};
		\node [style=none] (24) at (-1.5, -2.5) {};
		\node [style=none] (25) at (-5.5, 2.5) {};
		\node [style=none] (26) at (-5.5, 1.5) {};
		\node [style=none] (27) at (-5.5, 0.5) {};
		\node [style=none] (28) at (-1.5, -1.5) {};
		\node [style=none] (29) at (3.5, 0.5) {};
		\node [style=none] (30) at (3.5, -0.5) {};
		\node [style=none] (31) at (5.5, -0.5) {};
		\node [style=none] (32) at (5.5, 0.5) {};
		\node [style=none] (33) at (-3.5, 0.5) {};
		\node [style=none] (34) at (-3.5, -0.5) {};
		\node [style=none] (35) at (-1.5, -0.5) {};
		\node [style=none] (36) at (-1.5, 0.5) {};
		\node [style=none] (37) at (-3.5, 2.5) {};
		\node [style=none] (38) at (-3.5, 1.5) {};
		\node [style=none] (39) at (-1.5, 1.5) {};
		\node [style=none] (40) at (-1.5, 2.5) {};
	\end{pgfonlayer}
	\begin{pgfonlayer}{edgelayer}
		\draw (23.center) to (22.center);
		\draw (20.center) to (21.center);
		\draw (18.center) to (27.center);
		\draw (27.center) to (26.center);
		\draw (26.center) to (25.center);
		\draw [in=-90, out=90, looseness=0.50] (9.center) to (16.center);
		\draw [in=-90, out=90, looseness=0.50] (12.center) to (8.center);
		\draw (5.center) to (4.center);
		\draw (13.center) to (10.center);
		\draw [in=-90, out=90, looseness=0.50] (30.center) to (32.center);
		\draw [in=-90, out=90, looseness=0.50] (31.center) to (29.center);
		\draw [in=-90, out=90, looseness=0.50] (34.center) to (36.center);
		\draw [in=-90, out=90, looseness=0.50] (35.center) to (33.center);
		\draw [in=-90, out=90, looseness=0.50] (38.center) to (40.center);
		\draw [in=-90, out=90, looseness=0.50] (39.center) to (37.center);
		\draw (24.center) to (35.center);
		\draw (32.center) to (7.center);
	\end{pgfonlayer}
\end{tikzpicture}

%% file: figures/commutativity-directed-2.tikz
\begin{tikzpicture}
	\begin{pgfonlayer}{nodelayer}
		\node [style=none] (0) at (0, 0) {$=$};
		\node [style=large box] (1) at (-2.5, 1) {$U$};
		\node [style=none] (2) at (-1.5, -0.5) {};
		\node [style=large box] (3) at (-4.5, -1) {$U$};
		\node [style=none] (4) at (-3.5, 1.5) {};
		\node [style=none] (5) at (-3.5, 2.5) {};
		\node [style=none] (6) at (-5.5, 1.5) {};
		\node [style=none] (7) at (-5.5, 2.5) {};
		\node [style=none] (8) at (-3.5, -1.5) {};
		\node [style=none] (9) at (-3.5, -2.5) {};
		\node [style=none] (10) at (-1.5, -2.5) {};
		\node [style=none] (11) at (-1.5, -1.5) {};
		\node [style=none] (12) at (-5.5, -2.5) {};
		\node [style=none] (13) at (-1.5, 2.5) {};
		\node [style=none] (14) at (-1.5, 1.5) {};
		\node [style=none] (15) at (-1.5, 0.5) {};
		\node [style=none] (16) at (-5.5, -1.5) {};
		\node [style=large box] (17) at (4.5, -1) {$U$};
		\node [style=none] (18) at (5.5, -0.5) {};
		\node [style=large box] (19) at (2.5, 1) {$U$};
		\node [style=none] (20) at (3.5, -1.5) {};
		\node [style=none] (21) at (3.5, -2.5) {};
		\node [style=none] (22) at (5.5, -2.5) {};
		\node [style=none] (23) at (5.5, -1.5) {};
		\node [style=none] (24) at (1.5, -2.5) {};
		\node [style=none] (25) at (5.5, 2.5) {};
		\node [style=none] (26) at (5.5, 1.5) {};
		\node [style=none] (27) at (5.5, 0.5) {};
		\node [style=none] (28) at (1.5, -1.5) {};
		\node [style=none] (29) at (-3.5, 0.5) {};
		\node [style=none] (30) at (-3.5, -0.5) {};
		\node [style=none] (31) at (-5.5, -0.5) {};
		\node [style=none] (32) at (-5.5, 0.5) {};
		\node [style=none] (33) at (3.5, 0.5) {};
		\node [style=none] (34) at (3.5, -0.5) {};
		\node [style=none] (35) at (1.5, -0.5) {};
		\node [style=none] (36) at (1.5, 0.5) {};
		\node [style=none] (37) at (3.5, 2.5) {};
		\node [style=none] (38) at (3.5, 1.5) {};
		\node [style=none] (39) at (1.5, 1.5) {};
		\node [style=none] (40) at (1.5, 2.5) {};
	\end{pgfonlayer}
	\begin{pgfonlayer}{edgelayer}
		\draw (23.center) to (22.center);
		\draw (20.center) to (21.center);
		\draw (18.center) to (27.center);
		\draw (27.center) to (26.center);
		\draw (26.center) to (25.center);
		\draw [in=-90, out=90, looseness=0.50] (9.center) to (16.center);
		\draw [in=-90, out=90, looseness=0.50] (12.center) to (8.center);
		\draw (5.center) to (4.center);
		\draw (13.center) to (10.center);
		\draw [in=-90, out=90, looseness=0.50] (30.center) to (32.center);
		\draw [in=-90, out=90, looseness=0.50] (31.center) to (29.center);
		\draw [in=-90, out=90, looseness=0.50] (34.center) to (36.center);
		\draw [in=-90, out=90, looseness=0.50] (35.center) to (33.center);
		\draw [in=-90, out=90, looseness=0.50] (38.center) to (40.center);
		\draw [in=-90, out=90, looseness=0.50] (39.center) to (37.center);
		\draw (24.center) to (35.center);
		\draw (32.center) to (7.center);
	\end{pgfonlayer}
\end{tikzpicture}

%% file: figures/commutativity-directed-3.tikz
\begin{tikzpicture}
	\begin{pgfonlayer}{nodelayer}
		\node [style=large box] (0) at (-2.5, -1) {$U$};
		\node [style=none] (1) at (-1.5, 1.5) {};
		\node [style=none] (2) at (-1.5, 2.5) {};
		\node [style=none] (3) at (-3.5, 1.5) {};
		\node [style=none] (4) at (-3.5, 2.5) {};
		\node [style=none] (5) at (-1.5, -1.5) {};
		\node [style=none] (6) at (-1.5, -2.5) {};
		\node [style=none] (7) at (-3.5, -2.5) {};
		\node [style=none] (8) at (-3.5, -1.5) {};
		\node [style=none] (9) at (-1.5, 0.5) {};
		\node [style=none] (10) at (-1.5, -0.5) {};
		\node [style=none] (11) at (-3.5, -0.5) {};
		\node [style=none] (12) at (-3.5, 0.5) {};
		\node [style=large box] (13) at (-2.5, 1) {$U$};
		\node [style=large box] (14) at (2.5, 1) {$U$};
		\node [style=none] (15) at (3.5, -1.5) {};
		\node [style=none] (16) at (3.5, -2.5) {};
		\node [style=none] (17) at (1.5, -1.5) {};
		\node [style=none] (18) at (1.5, -2.5) {};
		\node [style=none] (19) at (3.5, 1.5) {};
		\node [style=none] (20) at (3.5, 2.5) {};
		\node [style=none] (21) at (1.5, 2.5) {};
		\node [style=none] (22) at (1.5, 1.5) {};
		\node [style=none] (23) at (3.5, -0.5) {};
		\node [style=none] (24) at (3.5, 0.5) {};
		\node [style=none] (25) at (1.5, 0.5) {};
		\node [style=none] (26) at (1.5, -0.5) {};
		\node [style=large box] (27) at (2.5, -1) {$U$};
		\node [style=none] (28) at (0, 0) {$=$};
	\end{pgfonlayer}
	\begin{pgfonlayer}{edgelayer}
		\draw [in=-90, out=90, looseness=0.50] (6.center) to (8.center);
		\draw [in=-90, out=90, looseness=0.50] (7.center) to (5.center);
		\draw (2.center) to (1.center);
		\draw [in=-90, out=90, looseness=0.50] (10.center) to (12.center);
		\draw [in=-90, out=90, looseness=0.50] (11.center) to (9.center);
		\draw (12.center) to (4.center);
		\draw [in=90, out=-90, looseness=0.50] (20.center) to (22.center);
		\draw [in=90, out=-90, looseness=0.50] (21.center) to (19.center);
		\draw (16.center) to (15.center);
		\draw [in=90, out=-90, looseness=0.50] (24.center) to (26.center);
		\draw [in=90, out=-90, looseness=0.50] (25.center) to (23.center);
		\draw (26.center) to (18.center);
	\end{pgfonlayer}
\end{tikzpicture}

%% file: figures/undirected-from-directed-1.tikz
\begin{tikzpicture}
	\begin{pgfonlayer}{nodelayer}
		\node [style=none] (0) at (-1.5, 1.5) {};
		\node [style=none] (1) at (-3.5, 1.5) {};
		\node [style=none] (2) at (-1.5, -1.5) {};
		\node [style=none] (3) at (-3.5, -1.5) {};
		\node [style=none] (4) at (-1.5, 0.5) {};
		\node [style=none] (5) at (-1.5, -0.5) {};
		\node [style=none] (6) at (-3.5, -0.5) {};
		\node [style=none] (7) at (-3.5, 0.5) {};
		\node [style=large box] (8) at (-2.5, 0) {$W$};
		\node [style=none] (9) at (0, 0) {$=$};
		\node [style=large box] (10) at (2.5, -1) {$U$};
		\node [style=none] (11) at (3.5, 1.5) {};
		\node [style=none] (12) at (3.5, 2.5) {};
		\node [style=none] (13) at (1.5, 1.5) {};
		\node [style=none] (14) at (1.5, 2.5) {};
		\node [style=none] (15) at (3.5, -1.5) {};
		\node [style=none] (16) at (3.5, -2.5) {};
		\node [style=none] (17) at (1.5, -2.5) {};
		\node [style=none] (18) at (1.5, -1.5) {};
		\node [style=none] (19) at (3.5, 0.5) {};
		\node [style=none] (20) at (3.5, -0.5) {};
		\node [style=none] (21) at (1.5, -0.5) {};
		\node [style=none] (22) at (1.5, 0.5) {};
		\node [style=large box] (23) at (2.5, 1) {$U$};
	\end{pgfonlayer}
	\begin{pgfonlayer}{edgelayer}
		\draw [in=-90, out=90, looseness=0.50] (16.center) to (18.center);
		\draw [in=-90, out=90, looseness=0.50] (17.center) to (15.center);
		\draw (12.center) to (11.center);
		\draw [in=-90, out=90, looseness=0.50] (20.center) to (22.center);
		\draw [in=-90, out=90, looseness=0.50] (21.center) to (19.center);
		\draw (22.center) to (14.center);
		\draw (5.center) to (2.center);
		\draw (6.center) to (3.center);
		\draw (7.center) to (1.center);
		\draw (0.center) to (4.center);
	\end{pgfonlayer}
\end{tikzpicture}

%% file: figures/undirected-from-directed-2.tikz
\begin{tikzpicture}
	\begin{pgfonlayer}{nodelayer}
		\node [style=none] (0) at (-4, 1.5) {};
		\node [style=none] (1) at (-6, 1.5) {};
		\node [style=none] (2) at (-4, -1.5) {};
		\node [style=none] (3) at (-6, -1.5) {};
		\node [style=none] (4) at (-4, 0.5) {};
		\node [style=none] (5) at (-4, -0.5) {};
		\node [style=none] (6) at (-6, -0.5) {};
		\node [style=none] (7) at (-6, 0.5) {};
		\node [style=large box] (8) at (-5, 0) {$W$};
		\node [style=none] (9) at (-4, 1.5) {};
		\node [style=none] (10) at (-4, 0.5) {};
		\node [style=none] (11) at (-6, 0.5) {};
		\node [style=none] (12) at (-6, 1.5) {};
		\node [style=none] (13) at (-4, -0.5) {};
		\node [style=none] (14) at (-4, -1.5) {};
		\node [style=none] (15) at (-6, -1.5) {};
		\node [style=none] (16) at (-6, -0.5) {};
		\node [style=none] (17) at (-2.5, 0) {$=$};
		\node [style=large box] (18) at (0, -0.5) {$U$};
		\node [style=none] (19) at (1, 2) {};
		\node [style=none] (20) at (1, 3) {};
		\node [style=none] (21) at (-1, 2) {};
		\node [style=none] (22) at (-1, 3) {};
		\node [style=none] (23) at (1, -1) {};
		\node [style=none] (24) at (1, -2) {};
		\node [style=none] (25) at (-1, -2) {};
		\node [style=none] (26) at (-1, -1) {};
		\node [style=none] (27) at (1, 1) {};
		\node [style=none] (28) at (1, 0) {};
		\node [style=none] (29) at (-1, 0) {};
		\node [style=none] (30) at (-1, 1) {};
		\node [style=large box] (31) at (0, 1.5) {$U$};
		\node [style=none] (32) at (1, 3) {};
		\node [style=none] (33) at (1, 2) {};
		\node [style=none] (34) at (-1, 2) {};
		\node [style=none] (35) at (-1, 3) {};
		\node [style=none] (36) at (1, -2) {};
		\node [style=none] (37) at (1, -3) {};
		\node [style=none] (38) at (-1, -3) {};
		\node [style=none] (39) at (-1, -2) {};
		\node [style=none] (40) at (2.5, 0) {$=$};
		\node [style=large box] (41) at (5, -1) {$U$};
		\node [style=none] (42) at (6, 1.5) {};
		\node [style=none] (43) at (6, 2.5) {};
		\node [style=none] (44) at (4, 1.5) {};
		\node [style=none] (45) at (4, 2.5) {};
		\node [style=none] (46) at (6, -1.5) {};
		\node [style=none] (47) at (4, -1.5) {};
		\node [style=none] (48) at (6, 0.5) {};
		\node [style=none] (49) at (6, -0.5) {};
		\node [style=none] (50) at (4, -0.5) {};
		\node [style=none] (51) at (4, 0.5) {};
		\node [style=large box] (52) at (5, 1) {$U$};
		\node [style=none] (53) at (6, 2.5) {};
		\node [style=none] (54) at (6, 1.5) {};
		\node [style=none] (55) at (4, 1.5) {};
		\node [style=none] (56) at (4, 2.5) {};
		\node [style=none] (57) at (6, -2.5) {};
		\node [style=none] (58) at (4, -2.5) {};
	\end{pgfonlayer}
	\begin{pgfonlayer}{edgelayer}
		\draw [in=-90, out=90, looseness=0.50] (10.center) to (12.center);
		\draw [in=-90, out=90, looseness=0.50] (11.center) to (9.center);
		\draw [in=-90, out=90, looseness=0.50] (14.center) to (16.center);
		\draw [in=-90, out=90, looseness=0.50] (15.center) to (13.center);
		\draw [in=-90, out=90, looseness=0.50] (24.center) to (26.center);
		\draw [in=-90, out=90, looseness=0.50] (25.center) to (23.center);
		\draw [in=-90, out=90, looseness=0.50] (28.center) to (30.center);
		\draw [in=-90, out=90, looseness=0.50] (29.center) to (27.center);
		\draw [in=-90, out=90, looseness=0.50] (33.center) to (35.center);
		\draw [in=-90, out=90, looseness=0.50] (34.center) to (32.center);
		\draw [in=-90, out=90, looseness=0.50] (37.center) to (39.center);
		\draw [in=-90, out=90, looseness=0.50] (38.center) to (36.center);
		\draw [in=-90, out=90, looseness=0.50] (49.center) to (51.center);
		\draw [in=-90, out=90, looseness=0.50] (50.center) to (48.center);
		\draw [in=-90, out=90, looseness=0.50] (54.center) to (56.center);
		\draw [in=-90, out=90, looseness=0.50] (55.center) to (53.center);
		\draw (47.center) to (58.center);
		\draw (57.center) to (46.center);
	\end{pgfonlayer}
\end{tikzpicture}

%% file: figures/undirected-from-directed-3.tikz
\begin{tikzpicture}
	\begin{pgfonlayer}{nodelayer}
		\node [style=none] (0) at (-2.5, 0) {$=$};
		\node [style=large box] (1) at (0, -1) {$U$};
		\node [style=none] (2) at (1, 1.5) {};
		\node [style=none] (3) at (1, 2.5) {};
		\node [style=none] (4) at (-1, 1.5) {};
		\node [style=none] (5) at (-1, 2.5) {};
		\node [style=none] (6) at (1, -1.5) {};
		\node [style=none] (7) at (1, -2.5) {};
		\node [style=none] (8) at (-1, -2.5) {};
		\node [style=none] (9) at (-1, -1.5) {};
		\node [style=none] (10) at (1, 0.5) {};
		\node [style=none] (11) at (1, -0.5) {};
		\node [style=none] (12) at (-1, -0.5) {};
		\node [style=none] (13) at (-1, 0.5) {};
		\node [style=large box] (14) at (0, 1) {$U$};
		\node [style=none] (15) at (2.5, 0) {$=$};
		\node [style=none] (16) at (6, 1.5) {};
		\node [style=none] (17) at (4, 1.5) {};
		\node [style=none] (18) at (6, -1.5) {};
		\node [style=none] (19) at (4, -1.5) {};
		\node [style=none] (20) at (6, 0.5) {};
		\node [style=none] (21) at (6, -0.5) {};
		\node [style=none] (22) at (4, -0.5) {};
		\node [style=none] (23) at (4, 0.5) {};
		\node [style=large box] (24) at (5, 0) {$W$};
	\end{pgfonlayer}
	\begin{pgfonlayer}{edgelayer}
		\draw [in=-90, out=90, looseness=0.50] (7.center) to (9.center);
		\draw [in=-90, out=90, looseness=0.50] (8.center) to (6.center);
		\draw (3.center) to (2.center);
		\draw [in=-90, out=90, looseness=0.50] (11.center) to (13.center);
		\draw [in=-90, out=90, looseness=0.50] (12.center) to (10.center);
		\draw (13.center) to (5.center);
		\draw (21.center) to (18.center);
		\draw (22.center) to (19.center);
		\draw (23.center) to (17.center);
		\draw (16.center) to (20.center);
	\end{pgfonlayer}
\end{tikzpicture}

%% file: figures/d3.tikz
\begin{tikzpicture}
	\begin{pgfonlayer}{nodelayer}
		\node [style=white dot] (0) at (-3, 0) {};
		\node [style=hadamard] (1) at (-3, 2) {};
		\node [style=white dot] (2) at (-3, 4) {};
		\node [style=hadamard] (3) at (-3, -2) {};
		\node [style=white dot] (4) at (-3, -4) {};
		\node [style=none] (5) at (0, 0) {$=$};
		\node [style=white dot] (6) at (3, 0) {};
		\node [style=none] (7) at (-3, 5) {};
		\node [style=none] (8) at (-3, -5) {};
		\node [style=none] (9) at (-2, 3.5) {$\alpha_{k+1}$};
		\node [style=none] (10) at (-2, -4.5) {$\alpha_{k-1}$};
		\node [style=none] (11) at (5.25, -0.5) {$\alpha_{k+1}+\alpha_{k-1}$};
		\node [style=none] (12) at (3, 1) {};
		\node [style=none] (13) at (3, -1) {};
		\node [style=none] (14) at (-2.5, -0.5) {$0$};
	\end{pgfonlayer}
	\begin{pgfonlayer}{edgelayer}
		\draw (7.center) to (2);
		\draw (2) to (1);
		\draw (1) to (0);
		\draw (0) to (3);
		\draw (3) to (4);
		\draw (4) to (8.center);
		\draw (6) to (13.center);
		\draw (6) to (12.center);
	\end{pgfonlayer}
\end{tikzpicture}

%% file: figures/d5.tikz
\begin{tikzpicture}
	\begin{pgfonlayer}{nodelayer}
		\node [style=white dot] (0) at (-3, 5) {};
		\node [style=hadamard] (1) at (0, 5) {};
		\node [style=none] (2) at (2, 5) {$=$};
		\node [style=none] (3) at (3.5, 5) {$0$};
		\node [style=white dot] (4) at (-3, 0) {};
		\node [style=white dot] (5) at (0, 0) {};
		\node [style=none] (6) at (2, 0) {$=$};
		\node [style=none] (7) at (3.5, 0) {$2$};
		\node [style=hadamard] (8) at (-1.5, 1.5) {};
		\node [style=hadamard] (9) at (-1.5, -1.5) {};
		\node [style=white dot] (10) at (-3, -5) {};
		\node [style=white dot] (11) at (0, -5) {$\pi$};
		\node [style=none] (12) at (2, -5) {$=$};
		\node [style=none] (13) at (3.5, -5) {$0$};
		\node [style=hadamard] (14) at (-1.5, -3.5) {};
		\node [style=hadamard] (15) at (-1.5, -6.5) {};
	\end{pgfonlayer}
	\begin{pgfonlayer}{edgelayer}
		\draw [bend left=90, looseness=1.50] (0) to (1);
		\draw [bend right=270, looseness=1.50] (1) to (0);
		\draw [bend left=45] (4) to (8);
		\draw [bend right=45] (5) to (8);
		\draw [bend left=45] (5) to (9);
		\draw [bend left=45] (9) to (4);
		\draw [bend left=45] (10) to (14);
		\draw [bend right=45] (11) to (14);
		\draw [bend left=45] (11) to (15);
		\draw [bend left=45] (15) to (10);
	\end{pgfonlayer}
\end{tikzpicture}

%% file: figures/d4.tikz
\begin{tikzpicture}
	\begin{pgfonlayer}{nodelayer}
		\node [style=dot] (0) at (1, 6) {$\pi$};
		\node [style=hadamard] (1) at (3, 6) {};
		\node [style=dot] (2) at (5, 6) {$\pi$};
		\node [style=hadamard] (3) at (7, 6) {};
		\node [style=hadamard] (4) at (-1, 6) {};
		\node [style=dot] (5) at (-3, 6) {$\pi$};
		\node [style=hadamard] (6) at (-5, 6) {};
		\node [style=dot] (7) at (-7, 6) {$\pi$};
		\node [style=none] (8) at (-8, 6) {};
		\node [style=none] (9) at (8, 6) {};
		\node [style=none] (10) at (-7, 3) {$=$};
		\node [style=gray dot] (11) at (-1, 3) {$\pi$};
		\node [style=white dot] (12) at (-3, 3) {$\pi$};
		\node [style=white dot] (13) at (1, 3) {$\pi$};
		\node [style=gray dot] (14) at (3, 3) {$\pi$};
		\node [style=none] (15) at (4, 3) {};
		\node [style=none] (16) at (-4, 3) {};
		\node [style=none] (17) at (-7, 0) {$=$};
		\node [style=none] (18) at (-6, 0) {$-1$};
		\node [style=white dot] (19) at (-1, 0) {$\pi$};
		\node [style=white dot] (20) at (-3, 0) {$\pi$};
		\node [style=gray dot] (21) at (1, 0) {$\pi$};
		\node [style=gray dot] (22) at (3, 0) {$\pi$};
		\node [style=none] (23) at (4, 0) {};
		\node [style=none] (24) at (-4, 0) {};
		\node [style=none] (25) at (-7, -3) {$\approx$};
		\node [style=white dot] (27) at (-1, -3) {$\pi$};
		\node [style=white dot] (28) at (-3, -3) {$\pi$};
		\node [style=gray dot] (29) at (1, -3) {$\pi$};
		\node [style=gray dot] (30) at (3, -3) {$\pi$};
		\node [style=none] (31) at (4, -3) {};
		\node [style=none] (32) at (-4, -3) {};
		\node [style=none] (33) at (-7, -6) {$=$};
		\node [style=none] (34) at (-1, -6) {};
		\node [style=none] (35) at (1, -6) {};
	\end{pgfonlayer}
	\begin{pgfonlayer}{edgelayer}
		\draw (9.center) to (3);
		\draw (3) to (2);
		\draw (2) to (1);
		\draw (1) to (0);
		\draw [in=360, out=180] (0) to (4);
		\draw (4) to (5);
		\draw (5) to (6);
		\draw (6) to (7);
		\draw (7) to (8.center);
		\draw (16.center) to (12);
		\draw (12) to (11);
		\draw (13) to (11);
		\draw (13) to (14);
		\draw (14) to (15.center);
		\draw (24.center) to (20);
		\draw (20) to (19);
		\draw (21) to (19);
		\draw (21) to (22);
		\draw (22) to (23.center);
		\draw (32.center) to (28);
		\draw (28) to (27);
		\draw (29) to (27);
		\draw (29) to (30);
		\draw (30) to (31.center);
		\draw (34.center) to (35.center);
	\end{pgfonlayer}
\end{tikzpicture}

%% file: figures/d6.tikz
\begin{tikzpicture}
	\begin{pgfonlayer}{nodelayer}
		\node [style=white dot] (0) at (-3, 3.5) {$\pi$};
		\node [style=hadamard] (1) at (0, 3.5) {};
		\node [style=none] (2) at (2, 3.5) {$=$};
		\node [style=none] (3) at (3.5, 3.5) {$\sqrt 2$};
		\node [style=white dot] (4) at (-3, -1.5) {$\pi$};
		\node [style=white dot] (5) at (0, -1.5) {$\pi$};
		\node [style=none] (6) at (2, -1.5) {$=$};
		\node [style=none] (7) at (3.5, -1.5) {$0$};
		\node [style=hadamard] (8) at (-1.5, 0) {};
		\node [style=hadamard] (9) at (-1.5, -3) {};
		\node [style=white dot] (10) at (-3, -6.5) {$\pi$};
		\node [style=white dot] (11) at (0, -6.5) {$\pi$};
		\node [style=none] (12) at (2, -7.25) {$=$};
		\node [style=none] (13) at (3.5, -7.25) {$\sqrt 2$};
		\node [style=hadamard] (14) at (-1.5, -5) {};
		\node [style=none] (16) at (-3, 8.5) {};
		\node [style=none] (17) at (0, 8.5) {};
		\node [style=none] (18) at (2, 8.5) {$=$};
		\node [style=none] (19) at (3.5, 8.5) {$2$};
		\node [style=hadamard] (20) at (-3, -8.25) {};
		\node [style=hadamard] (21) at (0, -8.25) {};
		\node [style=white dot] (22) at (-1.5, -9.75) {$\pi$};
	\end{pgfonlayer}
	\begin{pgfonlayer}{edgelayer}
		\draw [bend left=90, looseness=1.50] (0) to (1);
		\draw [bend right=270, looseness=1.50] (1) to (0);
		\draw [bend left=45] (4) to (8);
		\draw [bend right=45] (5) to (8);
		\draw [bend left=45] (5) to (9);
		\draw [bend left=45] (9) to (4);
		\draw [bend left=45] (10) to (14);
		\draw [bend right=45] (11) to (14);
		\draw [bend left=90, looseness=1.50] (16.center) to (17.center);
		\draw [bend right=270, looseness=1.50] (17.center) to (16.center);
		\draw (11) to (21);
		\draw (20) to (10);
		\draw [bend right=45] (20) to (22);
		\draw [bend right=45] (22) to (21);
	\end{pgfonlayer}
\end{tikzpicture}